\documentclass{article}

\usepackage{microtype}
\usepackage{graphicx}

\usepackage{booktabs} 
\usepackage{placeins}
\usepackage{subcaption}

\usepackage{hyperref}
\usepackage{adjustbox}

\usepackage[preprint]{icml2026}
\usepackage{enumitem}

\usepackage{amsmath}
\usepackage{algorithmic}
\usepackage{amssymb}
\usepackage{mathtools}
\usepackage{amsthm}
\usepackage{tikz}
\usetikzlibrary{positioning,fit,arrows.meta}
\usepackage[normalem]{ulem} 
\usepackage[capitalize,noabbrev]{cleveref}

\DeclareMathOperator*{\argmax}{\arg\max}

\theoremstyle{plain}
\newtheorem{theorem}{Theorem}[section]

\theoremstyle{definition}

\theoremstyle{remark}

\newtheorem{case}{Case}

\usepackage[textsize=tiny]{todonotes}

\icmltitlerunning{Differentiable Conformal Training for LLM Reasoning Factuality}

\DeclareFontFamily{U}{stix2bb}{}
\DeclareFontShape{U}{stix2bb}{m}{n} {<-> stix2-mathbb}{}
\NewDocumentCommand{\indicator}{}{\text{\usefont{U}{stix2bb}{m}{n}1}}
\usepackage{setspace}
\begin{document}
\newcommand{\squishlist}{
  \begin{list}{$\bullet$}{
    \setlength{\itemsep}{0pt}
    \setlength{\parsep}{3pt}
    \setlength{\topsep}{3pt}
    \setlength{\partopsep}{0pt}
    \setlength{\leftmargin}{1.5em}
    \setlength{\labelwidth}{1em}
    \setlength{\labelsep}{0.5em}
  }
}

\newcommand{\squishlisttwo}{
  \begin{list}{$\bullet$}{
    \setlength{\itemsep}{0pt}
    \setlength{\parsep}{0pt}
    \setlength{\topsep}{0pt}
    \setlength{\partopsep}{0pt}
    \setlength{\leftmargin}{2em}
    \setlength{\labelwidth}{1.5em}
    \setlength{\labelsep}{0.5em}
  }
}
\newcommand{\squishend}{
  \end{list}
}
\twocolumn[
\icmltitle{Differentiable Conformal Training for LLM Reasoning Factuality}



\begin{icmlauthorlist}
\icmlauthor{Nathan Hittesdorf}{yyy}
\icmlauthor{Marco Salzetta}{xxx}
\icmlauthor{Lu Cheng}{yyy}
\end{icmlauthorlist}

\icmlaffiliation{yyy}{Department of Computer Science, University of Illinois at Chicago, Chicago, United States}
\icmlaffiliation{xxx}{Department of Physics, University of Illinois at Urbana-Champaign, Urbana, United States}

\icmlcorrespondingauthor{Nathan Hittesdorf}{nhitt2@uic.edu}
\icmlcorrespondingauthor{Lu Cheng}{lucheng@uic.edu}

\icmlkeywords{Machine Learning, ICML, Conformal Prediction, Uncertainty Quantification, Reliability, Differentiable Optimization, Conformal Training}

\vskip 0.3in
]



\printAffiliationsAndNotice{} 

\label{sec:abstract}
\begin{abstract}
Large Language Models (LLMs) frequently hallucinate, limiting their reliability in critical applications. Conformal Prediction (CP) addresses this by calibrating error rates on held-out data to provide statistically valid confidence guarantees. Recent work extends CP to LLM factuality to filter out risky claims, ensuring that hallucination rates remain below a user-specified level (e.g., 10\%). While prior methods treat claims independently, Coherent Factuality \citep{rubin-toles2025conformal} extends to multi-step reasoning by representing outputs as dependency graphs and jointly validating claims with their logical ancestors. A key limitation is that Coherent Factuality is not differentiable, requiring hand-crafted scorers that at high reliability levels remove nearly 60\% of true claims. We introduce \textbf{Differentiable Coherent Factuality} (DCF), a fully differentiable relaxation that enables learning improved scorers while provably recovering the original algorithm's guarantees. Experiments on two benchmark reasoning datasets demonstrate DCF achieves up to \textbf{141\%} improvement in claim retention while maintaining reliability guarantees, representing a significant step towards reliable conformal LLM systems.
\end{abstract}

\section{Introduction}\label{sec:introduction}
LLMs are increasingly deployed in high-stakes decision-making, making output factuality critical. However, LLMs are prone to hallucinations---confidently generating false information. Extensive work addresses this, including fine-grained atomic evaluation \citep{min-etal-2023-factscore}, internal truthfulness detection \citep{azaria-mitchell-2023-internal}, and structured reasoning \citep{WeiChainofThought2022, wang2023selfconsistency}.\looseness=-1

Recent approaches adapt Conformal Prediction (CP) \citep{vovk2005algorithmic,AngelopoulosCPIntro2023,zhou2025conformal}---a distribution-free method for constructing prediction sets with guaranteed error rates---to provide statistical factuality guarantees for LLMs \citep{MohriHashimotoConformalFactuality2024, CherianLLMValidity2024}. These methods decompose outputs into atomic subclaims and assign each a \emph{risk score} measuring likelihood of incorrectness. CP calibrates a threshold $\tau$ on held-out data, retaining only subclaims with scores below $\tau$, guaranteeing that retained outputs are factual with probability at least $1-\alpha$ for user-specified error rate $\alpha$. Coherent Factuality (CF) \citep{rubin-toles2025conformal} extends this to multi-step reasoning via Approximate Deducibility Graphs (ADGs; \autoref{fig:example_graph}), where nodes are subclaims and edges encode logical dependencies. This captures contexts where claim validity depends on ancestor correctness, such as in mathematical reasoning.\looseness=-1

A critical limitation of these approaches is the trade-off between reliability and retention: achieving stronger guarantees (smaller $\alpha$) requires more aggressive filtering. This is fundamental to CP---to guarantee fewer errors, the method must become more conservative. The problem is particularly acute with hand-crafted risk scores. While learned scoring functions can improve retention in standard CP settings \citep{stutz2022learning, CherianLLMValidity2024}, CF relies on frequency-based self consistency scoring that cannot be optimized end-to-end. At high reliability levels ($\alpha < 0.1$), these hand-crafted approaches must remove up to 60\% of true claims to maintain coverage guarantees, severely limiting practical utility (\autoref{fig:retention_comparison}).\looseness=-1

To address this limitation, we introduce \textbf{Differentiable Coherent Factuality (DCF)}, a fully differentiable relaxation of CF that enables gradient-based optimization of risk scores. Prior work has made conformal prediction differentiable for classification \citep{stutz2022learning} and independent claim filtering \citep{CherianLLMValidity2024}, but these methods relax independent operations. CF's graph structure introduces tightly coupled discrete operations---threshold filtering, ancestor coherence enforcement, and argmax selection---that require joint relaxations preserving the algorithmic ordering (Section~\ref{sec:overview}). We develop such relaxations and prove they recover the original CF algorithm in the limit (Theorems~\ref{thm:calibration},~\ref{thm:prediction}), enabling training in a differentiable setting while preserving coverage guarantees at test time.\looseness=-1

Our contributions are:

\textbf{(1)} We prove theoretically and demonstrate empirically that CF's discrete operations can be faithfully modeled via differentiable relaxations (Theorems~\ref{thm:calibration},~\ref{thm:prediction}, Section~\ref{sec:validating-surrogate}).

\textbf{(2)} We show that the resulting framework enables end-to-end optimization of claim retention while preserving conformal coverage guarantees, achieving up to \textbf{141\%} improvement on MATH and \textbf{61\%} on FELM compared to frequency-based baselines (Sections~\ref{sec:comparing-baselines}).

\textbf{(3)} We provide in-depth interpretability analysis demonstrating that DCF learns to combine complementary signals more effectively than any individual feature (Section~\ref{sec:understanding-dcf}).

\section{Preliminaries}\label{sec:preliminaries}

\subsection{Conformal Prediction}\label{sec:cp-llms}
Conformal Prediction (CP) \citep{vovk2005algorithmic,AngelopoulosCPIntro2023} provides distribution-free coverage guarantees under exchangeability. Given inputs $x \in \mathcal{X}$ and ground-truth labels $y \in \mathcal{Y}$, a \emph{nonconformity score} $\nu_{\text{cp}}: \mathcal{X} \times \mathcal{Y} \to \mathbb{R}$ measures how atypical a prediction is relative to the true label. In the classic split CP setting \citep{AngelopoulosCPIntro2023}, we calibrate on a held-out set $\{(x_i, y_i)\}_{i=1}^n$ to compute threshold $\hat{\tau}_\alpha = \text{Quantile}_{\lceil(1-\alpha)(n+1)\rceil/n}(\{\nu_{\text{cp}}(x_i, y_i)\}_{i=1}^n)$. At test time, CP retains predictions with $\nu_{\text{cp}}(x,y) \leq \hat{\tau}_\alpha$, guaranteeing $1 - \alpha \leq \mathbb{P}(y_{n+1} \in C(x_{n+1})) \leq 1 - \alpha + \frac{1}{n+1}$.\looseness=-1

\textbf{Conformal factuality} \citep{MohriHashimotoConformalFactuality2024} adapts CP to provide statistical guarantees on LLM outputs (illustrated in \autoref{fig:llm_cf_pipeline}). Given an LLM response, the method: (1) decomposes it into atomic subclaims---self-contained factual statements that can be independently verified, (2) assigns each claim a \emph{risk score} $r_v$ measuring likelihood of incorrectness, and (3) retains only claims with $r_v \leq \hat{\tau}_\alpha$.\looseness=-1

\subsection{Coherent Factuality}\label{sec:coherent-factuality}
Coherent Factuality (CF)\footnote{We use CF to denote both the method and the property; context disambiguates.} \citep{rubin-toles2025conformal} extend CP to multi-step reasoning. While \citet{MohriHashimotoConformalFactuality2024} treats claims independently, CF recognizes that reasoning steps can only be evaluated within the context of preceding claims. For example, in proving $\sqrt{2}$ is irrational, the claim ``$p^2$ is even'' is only valid if the preceding claim ``$p^2 = 2q^2$'' is retained---CF enforces such dependencies.\looseness=-1

\textbf{Approximate Deducibility Graphs.}\label{cf:adg} CF represents outputs as ADGs $G = (V, E)$ where nodes are atomic subclaims and edges encode dependencies (\autoref{fig:example_graph}). A claim is \emph{coherently factual} if it and all ancestors are correct.

\textbf{Subgraph Generation.}\label{cf:subgraph} Given a risk score function $r: V \to \mathbb{R}$, CF generates candidate subgraphs $U_\tau$ for each threshold $\tau \in \mathcal{T}$ by: (1) selecting nodes with $r_v \leq \tau$, (2) removing nodes lacking ancestors, and (3) forming the induced subgraph.

\textbf{Nonconformity Score.}\label{cf:nonconformity} The nonconformity score is $\nu(X, Y, \mathcal{U}_\mathcal{T}) = \sup\{\tau : \text{all } U_{\tau'} \text{ with } \tau' \leq \tau \text{ are CF}\}$. The calibrated threshold $\hat{\tau}_\alpha$ is then obtained via split CP.

\textbf{Prediction.}\label{cf:prediction} At test time, the prediction set is $(U_{\text{filtered}}, \tau_{\text{filtered}}) = \argmax_{(U,\tau) : \tau < \hat{\tau}_\alpha} \tau$.

\textbf{Scoring Function.}\label{cf:scoring} CF uses hand-crafted frequency-based scoring: $s(v) = (1-\beta_{\text{mix}}) f_{\text{ind}}(v) + \beta_{\text{mix}} \cdot \text{med}\{f_{\text{ind}}(v') : v' \in \text{desc}(v)\}$, where $f_{\text{ind}}(v)$ is self-consistency frequency and $\beta_{\text{mix}} \in [0,1]$ is a mixing parameter.

\subsection{Conformal Training}\label{sec:conformal-training}
Conformal Training (ConfTr) \citep{stutz2022learning} makes CP differentiable for classification by relaxing two operations: set construction (sigmoid-smoothed thresholding) and calibration (differentiable sorting for quantile computation). These relaxations are \emph{independent}---each sample's confidence set is constructed in isolation. \citet{CherianLLMValidity2024} extend this to LLM factuality with independent claim filtering. In both cases, at test time the original non-smooth CP is applied, preserving coverage guarantees.\looseness=-1

\section{Method}
\subsection{Overview and Motivation}\label{sec:overview}

CF's hand-crafted frequency-based scoring leads to poor retention at strict $\alpha$ levels---removing up to 60\% of true claims. Conformal training offers a path forward by simulating the calibration and prediction pipeline during training to optimize scoring functions directly for claim retention. However, existing conformal training (Section~\ref{sec:conformal-training}) relaxes two independent operations---per-sample thresholding and quantile computation---where no prediction depends on another's outcome. CF's graph structure introduces three \emph{coupled} discrete operations: \textbf{(1)}~threshold filtering ($\indicator\{r_v \leq \tau\}$), \textbf{(2)}~ancestor coherence (retaining a node only if all ancestors pass), and \textbf{(3)}~argmax selection (a supremum over discrete sets shaped by prior steps). These form a cascade---each step's output depends on the previous step's discrete decisions---requiring joint relaxations that preserve this ordering.\looseness=-1

\textbf{Our approach.} We develop DCF: differentiable relaxations that enable gradient flow while preserving CF's algorithmic structure. Our approach mirrors CF's calibration and prediction procedure, replacing each discrete operation with a differentiable surrogate---sigmoids for indicators, products for conjunctions, and softmax for argmax. We then prove that these jointly recover the original CF algorithm in the limit (Theorems~\ref{thm:calibration},~\ref{thm:prediction}), theoretically grounding the use of DCF for optimizing claim retention.\footnote{Throughout Sections~\ref{sec:soft-membership}--\ref{sec:prediction}, we present formulas in log-space to match the implementation, where $\epsilon > 0$ is added to logarithm arguments for numerical stability.}\looseness=-1

\textbf{Proof challenges.} Proving that our relaxation recovers CF is non-trivial: the coupled temperature parameters controlling each relaxation must approach their limits in a specific order, since reversing the order corresponds to executing algorithmic steps out of sequence (e.g., taking a supremum before applying a sample cutoff). We address this by expressing all limits in terms of a single variable whose rate of approach encodes the correct ordering. An additional challenge arises when the supremum objective combines quantities on incomparable scales ($\tau$ values and violation indicators); we resolve this by restricting $\sup(\lambda\mathcal{T}) - \inf(\lambda\mathcal{T}) \leq 1$ to scale-match these terms. Full proofs appear in Appendix~\ref{sec:proofs}.\looseness=-1å

\subsection{Soft Membership}
\label{sec:soft-membership}

Calibration and prediction begin with the same two operations: threshold filtering and ancestor coherence (cf.\ \hyperref[cf:subgraph]{\textbf{Subgraph Generation}}). We relax these jointly into \emph{soft membership} probabilities $q_{v,\tau} \in (0,1)$ representing the degree to which node $v$ belongs to the subgraph at threshold $\tau$.\looseness=-1

\textbf{Soft Filtering.}
CF selects nodes with $r_v \leq \tau$. For an ADG $\mathcal{G} = (V, E)$ and differentiable scorer $\pi_\theta: \mathbb{R}^d \to \mathbb{R}$, we compute confidence scores $\pi_\theta(\mathbf{x}_v)$ and risk values $r_v = C - \pi_\theta(\mathbf{x}_v)$, where $C$ is a constant.\footnote{$C$ and margin $m$ are hyperparameters that aid comparison with CF.} Following \citet{rubin-toles2025conformal}, we construct threshold grid $\mathcal{T} = \{\tau_{\min}, r_{v_1}, \ldots, r_{v_n}, \tau_{\max}\}$ where $\tau_{\min} = \min_v r_v - m$ and $\tau_{\max} = \max_v r_v + m$.

The hard indicator for whether to retain a claim $v$ $\indicator\{r_v \leq \tau\}$ is non-differentiable. A natural relaxation is the sigmoid---a smooth approximation that outputs values in $(0,1)$ interpretable as retention probabilities $p_{v,\tau}$:
\begin{equation}
    \label{eqn:softkeep}
    p_{v,\tau} = \sigma\left(\frac{\tau - r_v}{T_p}\right), \quad \sigma(x) = \frac{1}{1+e^{-x}}.
\end{equation}
Temperature $T_p$ controls sharpness: as $T_p \to 0^+$, the sigmoid recovers the hard indicator.\looseness=-1

\paragraph{Soft Ancestor Coherence.}
CF removes nodes lacking ancestors---equivalently, a node is retained only if it \emph{and all its ancestors} pass filtering. This conjunction is naturally relaxed as a product: $\Pr(\text{node } v) = \prod_{u \in \text{Anc}(v) \cup \{v\}} \Pr(u)$. We implement this as a weighted geometric mean:\footnote{Geometric mean prevents chains from being punished for length; a standard product would dilute too aggressively.}
\begin{equation}
    \label{eqn:ancestor_coherence}
    \log q_{v,\tau} = \frac{\sum_{u \in \text{Anc}(v) \cup \{v\}} w_u \log p_{u,\tau}}
    {\sum_{u \in \text{Anc}(v) \cup \{v\}} w_u},
\end{equation}
where ancestors receive weight $w_u = \gamma$ and self receives $w_v = 1$. $\gamma > 0$ controls ancestor influence: $\gamma < 1$ attenuates, $\gamma = 1$ weights equally, $\gamma > 1$ amplifies. As a geometric mean of probabilities, $q_{v,\tau}$ approaches zero whenever any ancestor's retention probability approaches zero---preserving the conjunction.\looseness=-1

\subsection{Differentiable Calibration}
\label{sec:calibration}

Calibration computes a threshold $\hat{\tau}_\alpha$ guaranteeing $1-\alpha$ coverage (cf.\ \hyperref[cf:nonconformity]{\textbf{Nonconformity Score}}). For each calibration example, we find the maximum threshold at which the filtered graph contains no false claims---atomic subclaims that are hallucinated or factually incorrect. This defines the nonconformity score $\tilde{\tau}$. Taking a quantile across examples yields $\hat{\tau}_\alpha$.\looseness=-1

With soft membership probabilities $q_{v,\tau}$ in hand, we must relax the discrete supremum $\sup\{\tau : \text{all } U_{\tau'} \text{ with } \tau' \leq \tau \text{ are CF}\}$. This requires two steps: measuring how well each threshold filters false claims, then selecting the largest threshold with acceptable filtering.\looseness=-1

\textbf{Violation Measurement.}
We quantify how well threshold $\tau$ filters false claims via a \emph{violation score} $V_\tau \in [0,1]$, where higher values indicate the threshold is too permissive. Let $V^- = \{v : y_v = 0\}$ be the set of false claims. The validity score aggregates the probability of correctly filtering false claims:
\begin{equation}
    \log Q_{\tau} = \frac{1}{|V^-|} \sum_{v \in V^-} \log(1 - q_{v,\tau}).
\end{equation}
We convert validity to a violation measure:
\begin{equation}
    V_{\tau} = 1 - Q_{\tau}^{1/\tau_s},
\end{equation}
where $\tau_s > 0$ controls sharpness. As $\tau_s \to 0$ and $T_p \to 0$, this approaches hard violation detection.

\paragraph{Soft Supremum.}
The CF nonconformity score is a supremum: the largest $\tau$ such that all subgraphs up to $\tau$ are CF. A supremum is an argmax, which we relax via softmax over a utility balancing threshold magnitude against violations:
\begin{equation}
    s_{\tau} = \lambda \cdot \tau - V_{\tau},
\end{equation}
where $\lambda > 0$ controls the trade-off.\footnote{In practice, we min-max normalize both $\tau$ and $V_\tau$ to $[0,1]$ before computing the utility, so that $\lambda$ trades off quantities on comparable scales. To make the process differentiable, one may use softmin and softmax.} The soft supremum is:
\begin{equation}
    w_{\tau}^{\text{cal}} = \frac{\exp(\beta \cdot s_{\tau})}{\sum_{\tau' \in \mathcal{T}} \exp(\beta \cdot s_{\tau'})}, \quad
    \tilde{\tau} = \sum_{\tau \in \mathcal{T}} w_{\tau}^{\text{cal}} \cdot \tau.
\end{equation}
As $\beta \to \infty$, the softmax concentrates on the maximum-utility threshold, recovering the hard supremum. Intuitively, the soft supremum finds the largest threshold that avoids retaining false claims---exactly the nonconformity score needed for calibration.\looseness=-1

\paragraph{Conformal Quantile.}
The final calibration step computes $\hat{\tau}_\alpha$ as the $\lceil(1-\alpha)(n+1)\rceil/n$-quantile of the nonconformity scores $\{\tilde{\tau}_i\}_{i=1}^n$. In split CP, this is a simple order statistic; to enable gradient flow, we replace it with the differentiable soft quantile operator of \citet{grover2018stochastic}, which computes a smooth approximation $\hat{\tau}_\alpha = \text{SoftQuantile}(\{\tilde{\tau}_i\}, q, \rho)$ controlled by a sharpness parameter $\rho$. As $\rho \to \infty$, the soft quantile recovers the exact order statistic. This is the same mechanism used by ConfTr \citep{stutz2022learning} for differentiable calibration.\looseness=-1

\begin{theorem}[Calibration Convergence]
\label{thm:calibration}
As temperature parameters approach their limits ($T_p \to 0^+$, $\tau_s \to 0$, $\beta \to \infty$), the soft nonconformity score $\tilde{\tau}$ converges to the hard nonconformity score $\nu(\mathcal{U}_\mathcal{T})$ from CF. (Full statement and proof in Appendix~\ref{sec:proof-calibration}.)
\end{theorem}

\subsection{Differentiable Prediction}
\label{sec:prediction}

At test time, CF selects the maximum threshold below $\hat{\tau}_\alpha$ and returns the corresponding filtered subgraph (cf.\ \hyperref[cf:prediction]{\textbf{Prediction}}): $(U_*, \tau_*) = \argmax_{\tau < \hat{\tau}_\alpha} \tau$. We relax this constrained argmax while respecting the calibrated bound.\looseness=-1

We first compute soft membership probabilities $q_{v,\tau}$ via Equations~\eqref{eqn:softkeep}--\eqref{eqn:ancestor_coherence}. Then we relax the constrained argmax using softmax combined with a sigmoid gate.\looseness=-1

\paragraph{Gated Soft Argmax.}
The gate---another sigmoid relaxation of an indicator---smoothly transitions from 1 (below $\hat{\tau}_\alpha$) to 0 (above), enforcing the calibration bound:
\begin{equation}
    \label{eqn: w_tau unnorm}
    w_\tau^{\text{(unnorm)}} = \exp(\beta \cdot \tau) \cdot \sigma\left(\frac{\hat{\tau}_\alpha - \tau}{\tau_z}\right),
\end{equation}
where $\exp(\beta\cdot\tau)$ weights toward larger thresholds and the sigmoid gates out thresholds exceeding $\hat{\tau}_\alpha$. Temperature $\tau_z$ controls gate sharpness: as $\tau_z \to 0^+$, the gate recovers the hard constraint $\tau < \hat{\tau}_\alpha$. After normalization, final retention probabilities $q_v$ aggregate across thresholds:
\begin{equation} 
\label{eqn: w_tau}
    w_\tau = \frac{w_\tau^{\text{(unnorm)}}}{\sum_{\tau' \in \mathcal{T}} w_{\tau'}^{\text{(unnorm)}}}, \quad
    q_v = \sum_{\tau \in \mathcal{T}} w_\tau \cdot q_{v,\tau}.
\end{equation}

\begin{theorem}[Prediction Convergence]
\label{thm:prediction}
As temperature parameters approach their limits ($T_p \to 0^+$, $\tau_z \to 0^+$, $\beta \to \infty$), the soft retention probabilities $q_v$ converge to the hard CF prediction $U_{\text{filtered}}$. (Full statement and proof in Appendix~\ref{sec:proof-prediction}.)
\end{theorem}

\subsection{Training Objective}
\label{sec:training}

Following ConfTr \citep{stutz2022learning}, we simulate the full conformal pipeline during training. At each epoch, we split the training data into disjoint calibration and prediction subsets $\mathcal{D}_{\text{cal}}$ and $\mathcal{D}_{\text{pred}}$. Differentiable calibration (Section~\ref{sec:calibration}) is applied to $\mathcal{D}_{\text{cal}}$ to obtain $\hat{\tau}_\alpha$, and differentiable prediction (Section~\ref{sec:prediction}) generates soft retention probabilities $q_{i,v}$ on $\mathcal{D}_{\text{pred}}$. Because all operations are differentiable, gradients flow from the loss through quantile estimation and calibration back to the scorer parameters $\theta$.\looseness=-1

Following \citet{CherianLLMValidity2024}, we maximize the number of true claims retained:
\begin{equation}
    \mathcal{L}_{\text{retention}} = -\frac{1}{|\mathcal{D}_{\text{pred}}|} \sum_{i \in \mathcal{D}_{\text{pred}}} \sum_{v \in V_i} y_{i,v} \cdot q_{i,v},
\end{equation}
where $y_{i,v} \in \{0,1\}$ are ground-truth labels and $q_{i,v}$ are the coherence-adjusted soft retention probabilities from Equation~\eqref{eqn: w_tau}. While the loss function is shared with \citet{CherianLLMValidity2024}, there it optimizes independent claim filtering; here, gradients flow through the full graph-structured pipeline---calibration, ancestor coherence, and gated prediction---enabling the scorer to learn patterns that exploit CF's dependency structure. This directly targets the limitation identified in Section~\ref{sec:introduction}: learning scores that retain more correct claims under conformal guarantees.\looseness=-1

\subsection{Summary and Training Procedure}
\label{sec:summary}

Table~\ref{tab:cf_dcf_comparison} summarizes all transformations from CF to DCF. Theorems~\ref{thm:calibration} and~\ref{thm:prediction} establish that these relaxations jointly recover the original CF algorithm in the limit, ensuring that optimization of the differentiable surrogate corresponds to optimization of the discrete procedure. With convergence established, we formalize DCF into two differentiable subroutines---calibration (Algorithm~\ref{alg:calibration}) and prediction (Algorithm~\ref{alg:prediction})---which compose into the end-to-end training procedure (Algorithm~\ref{alg:training}). At test time, the learned scorer deploys in the original discrete CF algorithm, preserving coverage guarantees. Full algorithm details are in Appendix~\ref{sec:algorithms}.\looseness=-1

\begin{table*}[t]
\centering
\caption{Summary of transformations from CF (discrete) to DCF (differentiable).}
\label{tab:cf_dcf_comparison}
\footnotesize
\setlength{\tabcolsep}{3pt}
\renewcommand{\arraystretch}{1.05}
\begin{tabular}{l|c|c|c}
\toprule
\textbf{Operation} & \textbf{CF (Discrete)} & \textbf{DCF (Differentiable)} & \textbf{Limit} \\
\midrule
Threshold filtering & $\indicator\{r_v \leq \tau\}$ & $\sigma\!\left(\frac{\tau - r_v}{T_p}\right)$ & $T_p \to 0^+$ \\
Ancestor coherence & $v \in V_\text{true} \!\Rightarrow\! \text{Anc}(v) \subseteq V_\text{true}$ & $q_{v,\tau} = \text{GeomMean}(\{p_{u,\tau}\}_{u \in \text{Anc}(v) \cup \{v\}})$ & (conjunction) \\
\midrule
Violation detection & $U_\tau$ is CF or not & $V_\tau = 1 - Q_\tau^{1/\tau_s}$ & $\tau_s \to 0$ \\
Nonconf.\ score & $\sup\{\tau : U_\tau \text{ is CF}\}$ & $\sum_\tau w_\tau \cdot \tau$ & $\beta \to \infty$ \\
Quantile & $\lceil(1{-}\alpha)(n{+}1)\rceil$-th order stat. & Soft quantile with strength $\rho$ & $\rho \to \infty$ \\
\midrule
Prediction & $\arg\max_{\tau < \hat{\tau}_\alpha} \tau$ & $\sum_\tau w_\tau \cdot q_{v,\tau}$, \; $w_\tau^{\text{(unnorm)}} \propto \exp(\beta\tau) \cdot \sigma\!\left(\frac{\hat{\tau}_\alpha - \tau}{\tau_z}\right)$ & $\beta \to \infty$, $\tau_z \to 0$ \\
\bottomrule
\end{tabular}
\end{table*}

\section{Experiments}\label{sec:experiments}
As established in Section~\ref{sec:introduction}, hand-crafted scoring functions remove up to 60\% of true claims at high reliability levels. Our experiments evaluate whether DCF's learned scoring can improve this retention-guarantee trade-off. We focus on $\alpha \in [0.01, 0.10]$---corresponding to 90--99\% CF guarantees---where the trade-off is most acute and reliability matters most. Our experiments address three questions:
\begin{itemize}[leftmargin=*,nosep]
    \item \textbf{RQ1 (Validity)}: Does DCF faithfully approximate CF to enable valid gradient-based optimization?
    \item \textbf{RQ2 (Performance)}: Does DCF retain more true claims than the SOTA methods while maintaining coverage guarantees?
    \item \textbf{RQ3 (Understanding)}: Does learning to \emph{combine} features provide value beyond any single feature?
\end{itemize}

\subsection{Experimental Setup}\label{sec:experimental-setup}

\subsubsection{Scorer Architecture}
\label{sec:scorer}
We implement $\pi_\theta$ as logistic regression over claim features to avoid overfitting and enable interpretability analysis.\looseness=-1

\subsubsection{Datasets}
\label{sec:data_setup}
We follow \cite{rubin-toles2025conformal} to use the following two benchmark datasets:
(1) \textbf{MATH} \citep{hendrycks2021measuring}: 202 competition-level mathematics problems. We use GPT-5-mini \citep{openai_gpt5mini_2025} to generate solutions decomposed into atomic subclaims forming ADGs (\autoref{fig:example_graph}). Two annotators with strong math background performed dual-annotation to verify dependencies and label correctness. Graphs average 7.3 claims and 7.3 edges per problem, with 81.2\% fully correct.\looseness=-1
(2) \textbf{FELM} \citep{chen2023felm}: 710 problems across four domains: reasoning (207), math (194), world knowledge (184), and science (125), with subclaim-level decompositions and correctness annotations provided by the dataset authors. We exclude writing, which lacks structured logical dependencies. FELM exhibits simpler reasoning chains (4.0 claims, 2.8 edges per problem) with 62.4\% fully correct.\looseness=-1

\textbf{Features.} To provide training signals beyond frequency for the training of our scorer model, we extract 30 claim-level features organized into four categories (\citet{CherianLLMValidity2024} similarly extract claim features for learned scoring): (1)~base scoring (self-consistency frequency), (2)~semantic coherence (logical consistency with ancestors, inference gap detection), (3)~domain indicators (12 binary flags for mathematical topics), and (4)~graph metrics (PageRank, betweenness, reachability computed via NetworkX). Semantic features use the same GPT-5-mini model as frequency scoring, ruling out gains from stronger model access. For FELM, we exclude MATH-specific domain indicators and use two feature configurations: 7 features for $\alpha \leq 0.08$ and 20 features for $\alpha \geq 0.09$ (selected by correlation with ground-truth labels). See Appendix~\ref{sec:feature_details} for complete feature descriptions.\looseness=-1

\subsubsection{Evaluation Protocol}

We use 20-fold cross-validation to reduce variance in our estimates. Each fold partitions data into three subsets with distinct roles: (1) \textbf{Training set} (70\%): Used to learn the scorer $\pi_\theta$, further split into calibration and prediction subsets for differentiable training (\autoref{fig:training_flow} in Appendix \ref{sec:training_flow})
(2) \textbf{Validation set} (15\%): For early stopping during training
(3) \textbf{Test set} (15\%): For final evaluation.
At test time, the learned scorer deploys in the original CF algorithm, preserving theoretical guarantees.\looseness=-1

\subsubsection{Baselines}
We compare DCF against state-of-the-art conformal LLM factuality methods, each ablating a key component of DCF:
\begin{itemize}[leftmargin=*,nosep]
    \item \textbf{Coherent Factuality (CF)} \citep{rubin-toles2025conformal}: Hand-crafted frequency-based scoring with graph structure. \emph{Does learning improve over hand-crafted scoring?}
    \item \textbf{Independent Factuality} \citep{MohriHashimotoConformalFactuality2024}: Frequency-based scoring without graph structure---claims treated independently. \emph{Does modeling dependencies help?}
    \item \textbf{Boosted Independent} \citep{CherianLLMValidity2024}: Learned scoring with the same features as DCF, but without graph structure. \emph{Does learning alone suffice, or is graph structure necessary?}
    \item \textbf{XGBoost + Conformal}: XGBoost trained independently for optimizing CF accuracy on the same features, then plugged into CF as the scorer. \emph{Can a standard classifier replace DCF's differentiable training?}
\end{itemize}

\subsubsection{Metrics}

Theorems~\ref{thm:calibration} and~\ref{thm:prediction} establish that DCF recovers CF in the limit. To verify this empirically---confirming that soft operations approximate hard operations well enough to train $\pi_\theta$ (RQ1)---we report \textbf{correlation metrics} (Pearson $r$, MAE) between soft and hard quantities.
For evaluating performance (RQ2, RQ3), we follow standard CP evaluation \citep{AngelopoulosCPIntro2023} and LLM factuality evaluation \cite{rubin-toles2025conformal}:
(1) \textbf{Coverage}: Proportion of test examples where retained claims are CF; target is $1-\alpha$. (2) \textbf{Retention}: Mean claims retained per problem (or equivalently, proportion of claims retained), measuring informativeness of filtered outputs.

\subsection{RQ1: Validating Surrogate Behavior}\label{sec:validating-surrogate}

We empirically verify Theorems~\ref{thm:calibration} and~\ref{thm:prediction} on MATH, focusing on problems with at least one incorrect claim where calibration is non-trivial (error-free problems yield constant thresholds uninformative for validation). We use MATH because (1) we performed manual annotation with direct control over label quality, and (2) its denser graphs (avg.\ 7.3 claims, 7.3 edges vs.\ FELM's 4.0 claims, 2.8 edges) make ancestor constraints more meaningful.\looseness=-1

\subsubsection{Calibration Validation}

Training uses finite temperatures (Appendix~\ref{sec:hyperparam-config}), so we validate empirically that these settings yield behavior sufficiently close to CF. Figure~\ref{fig:calibration_individual} shows strong correlation between soft and hard nonconformity scores ($r=0.921$, MAE$=0.34$). Figure~\ref{fig:calibration_quantile} demonstrates near-perfect threshold agreement across $\alpha \in [0.01, 0.15]$ ($r=0.976$, MAE$=0.12$, 3\% relative error), confirming the soft quantile preserves calibration properties.\footnote{Since 18.8\% of MATH problems contain errors, thresholds saturate at $\tau_{\max}$ for $\alpha \gtrsim 0.19$. We evaluate at $\alpha \leq 0.15$, below saturation.}\looseness=-1

\begin{figure}[!htbp]
\centering
\begin{subfigure}[b]{0.48\columnwidth}
    \centering
    \includegraphics[width=\textwidth]{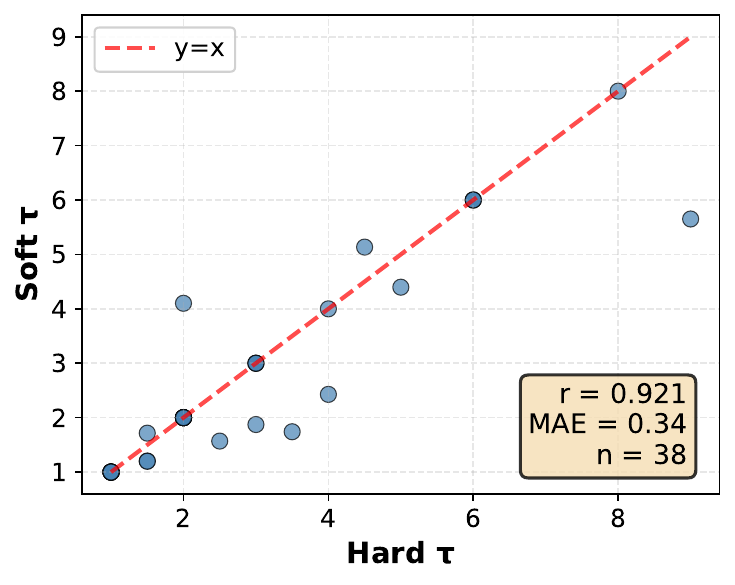}
    \caption{Individual nonconf.\ scores.}
    \label{fig:calibration_individual}
\end{subfigure}
\begin{subfigure}[b]{0.48\columnwidth}
    \centering
    \includegraphics[width=\textwidth]{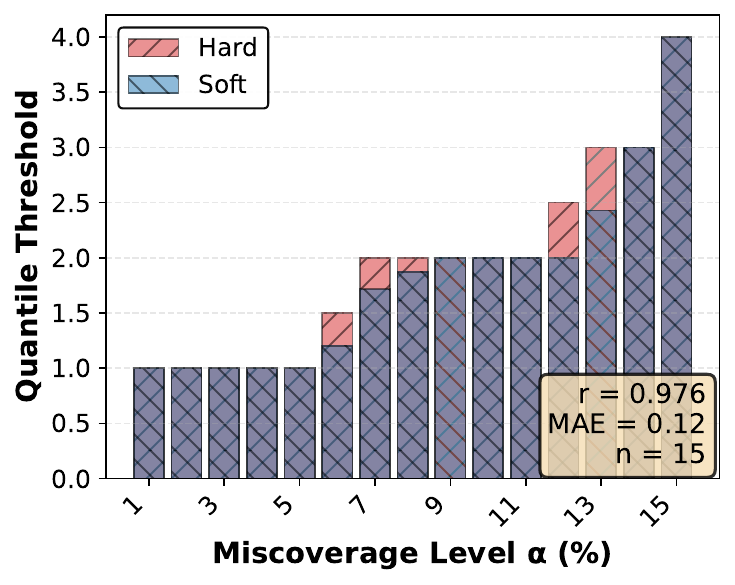}
    \caption{Agg.\ thresholds across $\alpha$.}
    \label{fig:calibration_quantile}
\end{subfigure}
\caption{Soft vs.\ hard calibration validation.}
\label{fig:calibration}
\end{figure}

\subsubsection{Prediction Validation}

Table~\ref{tab:soft_hard_agreement} shows claim-level agreement across 20-fold CV (14,600 predictions). Agreement ranges from 90.2--100\%, with perfect agreement at $\alpha \leq 0.03$ and minor discrepancies at threshold boundary regions. This validates Theorem~\ref{thm:prediction}.\looseness=-1

\begin{table}[!htbp]
\centering
\caption{Soft-hard prediction agreement across $\alpha$ values.}
\label{tab:soft_hard_agreement}
\footnotesize
\setlength{\tabcolsep}{1.7pt}
\renewcommand{\arraystretch}{0.95}
\begin{tabular}{c|rrrr|c}
\toprule
$\alpha$ & Both Incl. & Soft Only & Hard Only & Both Excl. & Agree \\
\midrule
0.01--0.03 & 4791 &   0 &   0 & 9809 & 100.0\% \\
0.04 & 4791 &   0 &  30 & 9779 & 99.8\% \\
0.05 & 4813 &   0 & 902 & 8885 & 93.8\% \\
0.06 & 5428 &   2 & 888 & 8282 & 93.9\% \\
0.07 & 5848 &   4 & 668 & 8080 & 95.4\% \\
0.08 & 6206 &   8 & 1232 & 7154 & 91.5\% \\
0.09 & 6923 &  12 & 1416 & 6249 & 90.2\% \\
0.10 & 7455 &  54 & 994 & 6097 & 92.8\% \\
\bottomrule
\end{tabular}
\end{table}

\paragraph{Summary.} Both calibration ($r > 0.92$) and prediction ($>90\%$ agreement) validations confirm that DCF with practical hyperparameters faithfully approximates CF, ensuring training gradients reflect the true discrete objective.\looseness=-1

\subsection{RQ2: Performance Comparisons}\label{sec:comparing-baselines}

Having validated DCF as a faithful surrogate (RQ1), we evaluate whether the differentiable framework learns improved scoring functions. We compare DCF against the baselines introduced in Section~\ref{sec:experimental-setup}, isolating whether improvements stem from learning, from modeling dependencies, or from their combination.\looseness=-1

\textbf{Main Results.}
Figure~\ref{fig:main_comparison} presents retention and coverage across $\alpha \in [0.01, 0.10]$ on MATH and FELM respectively.\footnote{At $\alpha \in {0.01, 0.02}$ on MATH, 81.2\% of problems are fully correct, leaving few calibration examples with errors. This makes quantile estimation unstable, causing the estimator to reject all claims. We report these values but focus analysis on $\alpha \geq 0.03$.} On MATH, DCF meets the coverage target for $\alpha \in {0.05, 0.06, 0.07, 0.08}$, achieving 60--105\% retention improvements over CF. At $\alpha = 0.03$ (97\% reliability), DCF more than doubles retention (+141\%, 1.76 vs.\ 0.73 claims) with coverage within 0.5pp of target. On FELM, DCF meets the coverage target at 9 of 10 $\alpha$ values and achieves higher retention than CF at 6 of those 9, with gains concentrated at lower $\alpha$ (0.01--0.05) and higher $\alpha$ (0.09--0.10). At mid-range $\alpha$ (0.06--0.08), CF achieves higher retention; however, DCF's coverage at these levels exceeds the target by 0.9--1.6pp (e.g., 93.64\% vs.\ 92\% target at $\alpha=0.08$), suggesting room for trading excess coverage for improved retention. Complete numerical results appear in Tables~\ref{tab:math_results} and~\ref{tab:felm_results} (Appendix~\ref{sec:experimental_details}).\looseness=-1

\paragraph{Baseline Ordering Validates Design.}
DCF consistently outperforms all baselines, including the three state-of-the-art conformal factuality methods. The ordering (DCF $>$ CF $>$ Independent $>$ Boosted Independent $\gg$ XGBoost) confirms that learning improves over hand-crafted scoring (DCF vs.\ CF), that modeling dependencies adds value even without learning (CF vs.\ Independent), and that learning without graph structure actually hurts---Boosted Independent underperforms non-learned Independent while using identical frequency-score. XGBoost trained independently for claim-level accuracy achieves near-zero retention at tight coverage levels, showing that a standard classifier cannot replace DCF's differentiable conformal training. Together, these results demonstrate that DCF's gains require \emph{both} learning and graph structure---neither component alone is sufficient---and that learning must occur \emph{through} the conformal objective.\looseness=-1
\begin{figure*}[htbp]
\centering
\begin{subfigure}[b]{0.24\textwidth}
    \centering
    \includegraphics[width=\textwidth]{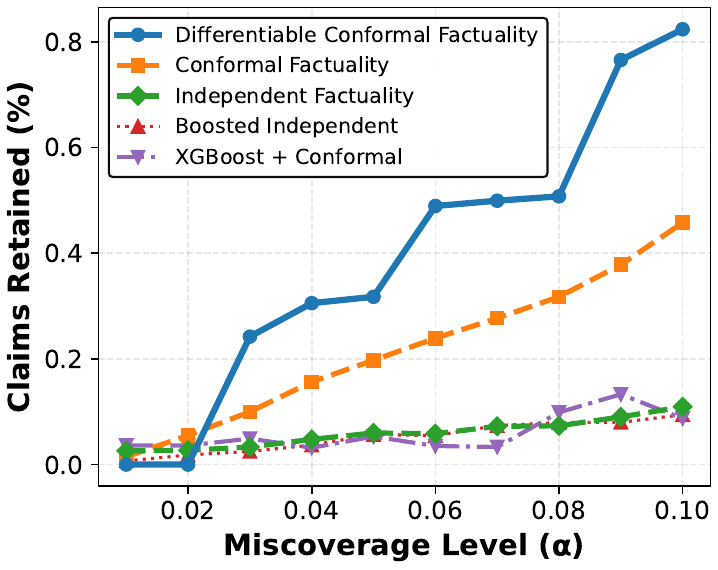}
    \caption{MATH Retention.}
    \label{fig:retention_comparison}
\end{subfigure}%
\hfill
\begin{subfigure}[b]{0.24\textwidth}
    \centering
    \includegraphics[width=\textwidth]{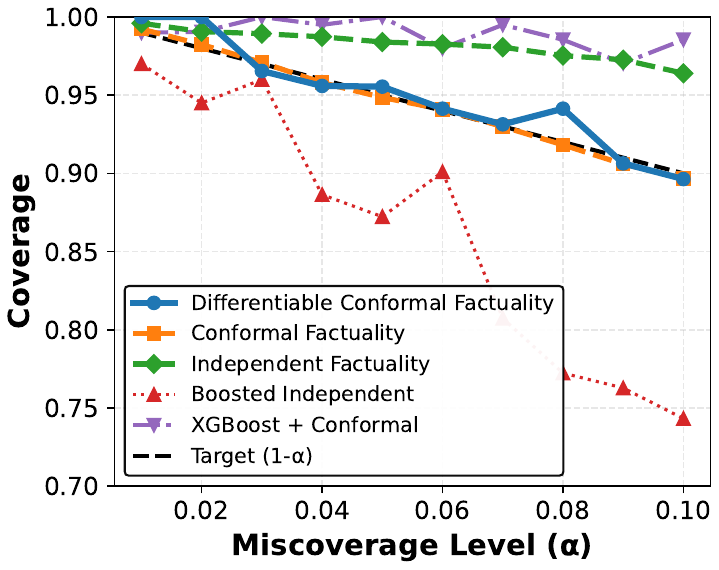}
    \caption{MATH Coverage.}
    \label{fig:coverage_comparison}
\end{subfigure}%
\hfill
\begin{subfigure}[b]{0.24\textwidth}
    \centering
    \includegraphics[width=\textwidth]{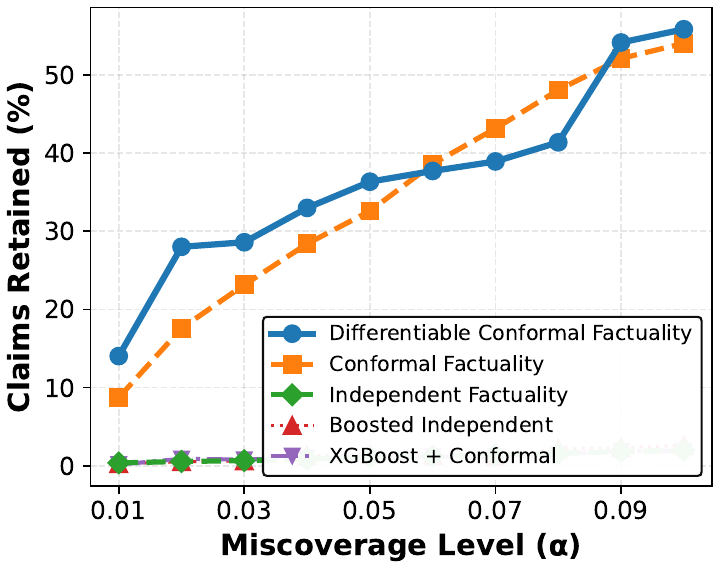}
    \caption{FELM Retention.}
    \label{fig:felm_retention}
\end{subfigure}%
\hfill
\begin{subfigure}[b]{0.24\textwidth}
    \centering
    \includegraphics[width=\textwidth]{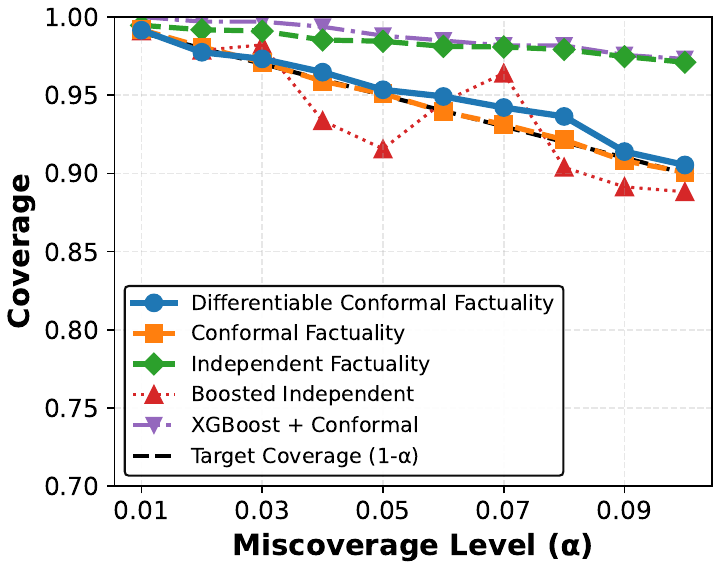}
    \caption{FELM Coverage.}
    \label{fig:felm_coverage}
\end{subfigure}
\caption{DCF vs.\ baselines on MATH (a--b) and FELM (c--d) across $\alpha \in [0.01, 0.10]$.}
\label{fig:main_comparison}
\label{fig:felm_comparison}
\end{figure*}

\paragraph{Summary.} DCF consistently outperforms SOTA across both datasets and reliability levels by directly optimizing CF's conformal objective with learned, graph-aware scoring.\looseness=-1

\subsection{RQ3: Understanding DCF}\label{sec:understanding-dcf}

We now ask whether DCF's gains stem from learning to \emph{combine} features, or simply from accessing more discriminative individual features. CF can use any scalar feature as its risk score---the original method uses frequency, but graph metrics or positional features could substitute. We test whether any single feature, optimally configured, can match DCF.\looseness=-1

\subsubsection{Single-Feature Ablation}

Using SHAP analysis (Appendix~\ref{sec:shap_analysis}), we identify DCF's most influential features: \texttt{nx\_reachability} (graph connectivity), \texttt{claim\_index} (position in reasoning chain), and \texttt{frequency-score}. We construct single-feature CF baselines using each as the risk score, with $\beta_{\text{mix}}$ optimized per-$\alpha$ via grid search.

\begin{figure}[htbp]
\centering
\begin{subfigure}[b]{0.50\columnwidth}
    \centering
    \includegraphics[width=\textwidth]{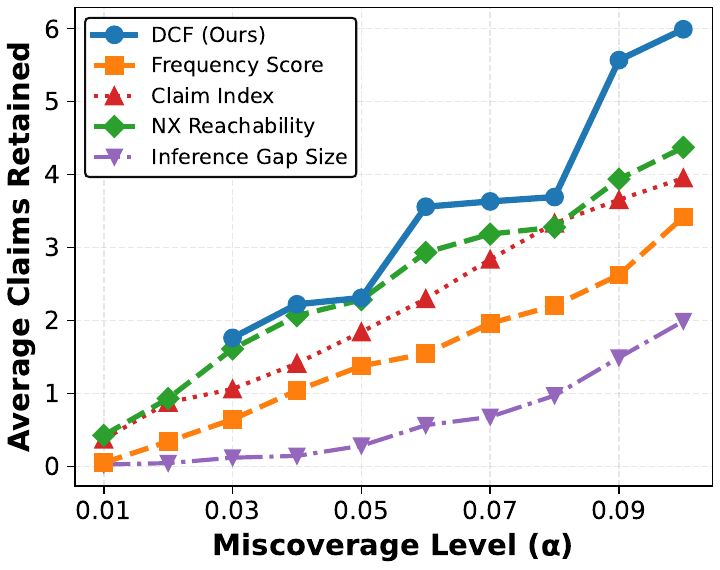}
    \caption{MATH Retention.}
    \label{fig:ablation_retention}
\end{subfigure}%
\hfill
\begin{subfigure}[b]{0.50\columnwidth}
    \centering
    \includegraphics[width=\textwidth]{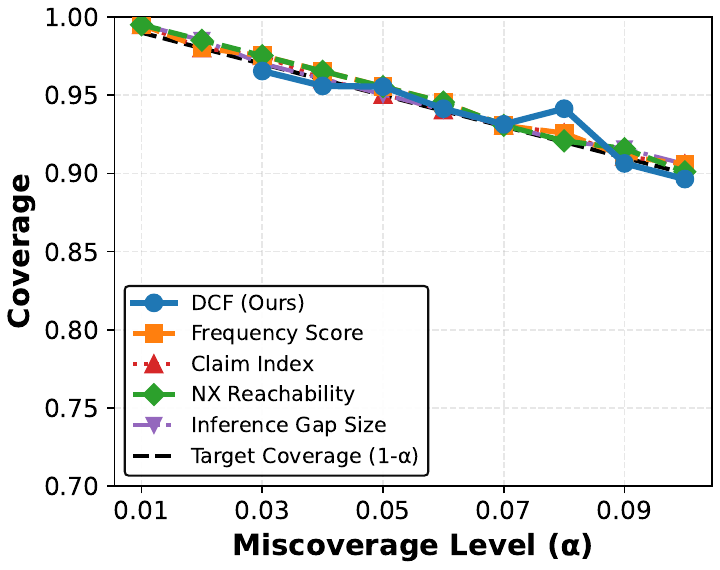}
    \caption{MATH Coverage.}
    \label{fig:ablation_coverage}
\end{subfigure}

\vspace{0.1cm}

\begin{subfigure}[b]{0.50\columnwidth}
    \centering
    \includegraphics[width=\textwidth]{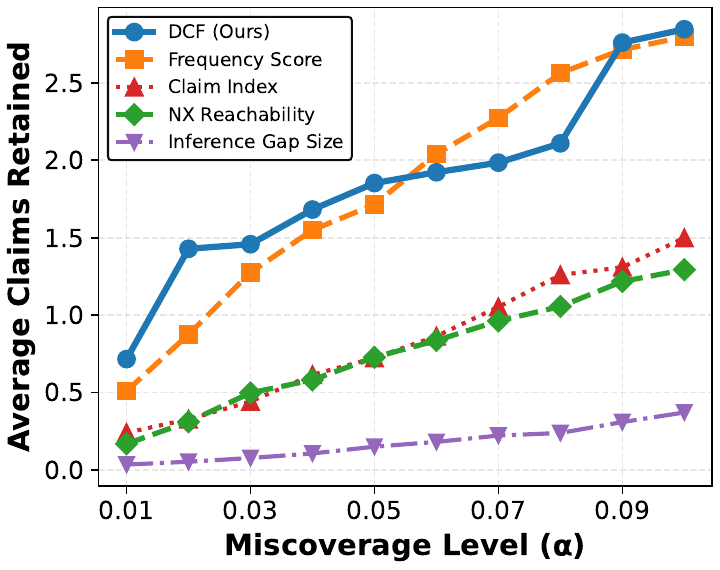}
    \caption{FELM Retention.}
    \label{fig:ablation_retention_felm}
\end{subfigure}%
\hfill
\begin{subfigure}[b]{0.50\columnwidth}
    \centering
    \includegraphics[width=\textwidth]{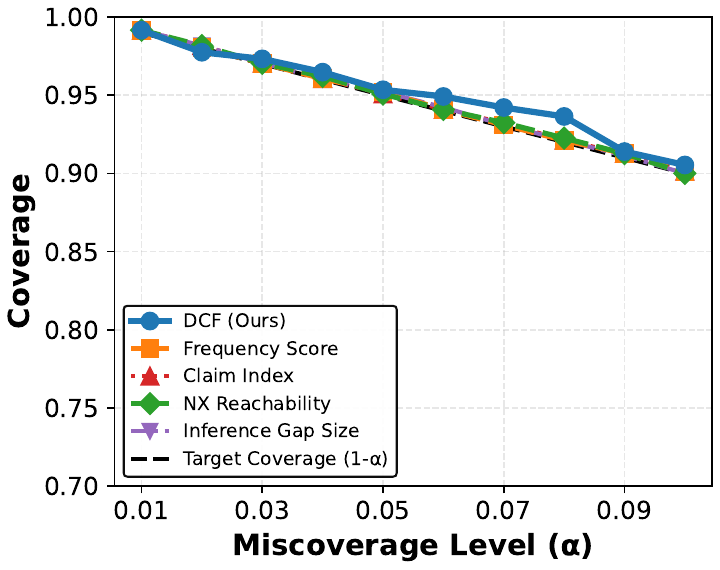}
    \caption{FELM Coverage.}
    \label{fig:ablation_coverage_felm}
\end{subfigure}
\caption{DCF vs.\ single-feature CF baselines on MATH (top) and FELM (bottom). Baselines use one feature with optimized $\beta_{\text{mix}}$.}
\label{fig:ablation}
\end{figure}

\autoref{fig:ablation} shows DCF consistently outperforms single-feature baselines on both datasets. On MATH, DCF outperforms the best single feature (\texttt{nx\_reachability}) by 7--50\% across $\alpha$ values. On FELM, DCF outperforms frequency-score at most $\alpha$ values (by up to 64\%); at mid-range $\alpha$ where frequency wins, DCF maintains excess coverage (Section~\ref{sec:comparing-baselines}). Notably, the best single feature differs by dataset---graph connectivity dominates on MATH's complex reasoning chains, while frequency suffices more on FELM's simpler structures---yet DCF matches or exceeds both.\looseness=-1

\subsubsection{Case Study}
\label{sec:case_studies}

To illustrate how combining features improves retention, consider a MATH problem at $\alpha=0.06$ where all 12 claims are correct. The CF baseline (frequency scoring with optimized $\beta_{\text{mix}}$) retains 0/12 claims; DCF retains 9/12.\looseness=-1

\begin{figure}[htbp]
\centering
\scalebox{0.7}{
\begin{tikzpicture}[
    node distance=0.3cm and 0.5cm,
    claim/.style={circle, draw, minimum size=0.65cm, font=\small, fill=green!20, line width=1.5pt},
    notretained/.style={circle, draw, minimum size=0.65cm, font=\small, fill=gray!30, line width=1.5pt, draw=gray!70},
    highlight/.style={circle, draw, minimum size=0.65cm, font=\small, fill=yellow!40, line width=2.5pt, draw=orange},
    arrow/.style={->, >=stealth, thick}
]

\node[text width=7cm, font=\normalsize, align=justify, draw, fill=blue!5, anchor=north west] (problem) at (-5,0) {
    \textbf{Problem:} The difference of the roots of the quadratic equation $x^2 + bx + c = 0$ is $|b - 2c|$. If $c \neq 0$, find $c$ in terms of $b$.\\[2pt]
    \textbf{Answer:} $c = b - 1$
};

\node[below=0.3cm of problem.south west, anchor=north west, text width=7cm, font=\small, align=left, draw, fill=white] (claims) {
    \textbf{Claims (all TRUE):}\\
    0: Let roots be $r_1, r_2$\\
    1: By Vieta, $r_1 + r_2 = -b$\\
    2: By Vieta, $r_1 r_2 = c$\\
    3: $(r_1 - r_2)^2 = (r_1 + r_2)^2 - 4r_1r_2$\\
    4: Using above: $(r_1-r_2)^2 = b^2-4c$\\
    5: Given $|r_1-r_2| = |b-2c|$\\
    6: Squaring: $b^2-4c = (b-2c)^2$\\
    7: $(b-2c)^2 = b^2-4bc+4c^2$\\
    8: Rearranging: $0=-4bc+4c^2+4c$ \textcolor{orange}{$\leftarrow$ freq=0}\\
    9: $0 = 4c(c+1-b)$\\
    10: From above: $c=0$ or $c=b-1$\\
    11: Since $c \neq 0$, $c=b-1$
};

\node[below=0.3cm of claims.south west, anchor=north west, text width=7cm, font=\normalsize, align=justify, draw, fill=green!5] (summary) {
    \textbf{Prediction set (majority vote):}\\
    \textcolor{green!70!black}{Learned: 9/12 claims retained}\\
    \textcolor{red!70!black}{Baseline: 0/12 claims retained}\\
    \textcolor{orange}{Highlighted: Claim 8 (freq=0)}
};

\node[claim] (c0) at (4.0, 0) {0};
\node[claim, below left=0.25cm and 0.2cm of c0] (c1) {1};
\node[claim, below right=0.25cm and 0.2cm of c0] (c2) {2};
\node[claim, below=0.5cm of c0] (c3) {3};
\node[claim, below=0.3cm of c3] (c4) {4};
\node[claim, below=0.3cm of c4] (c5) {5};
\node[claim, below=0.3cm of c5] (c6) {6};
\node[claim, below left=0.25cm and 0.2cm of c6] (c7) {7};
\node[highlight, below right=0.25cm and 0.2cm of c6] (c8) {8};
\node[notretained, below=0.5cm of c6] (c9) {9};
\node[notretained, below=0.3cm of c9] (c10) {10};
\node[notretained, below=0.3cm of c10] (c11) {11};

\draw[arrow] (c0) -- (c1);
\draw[arrow] (c0) -- (c2);
\draw[arrow] (c1) -- (c4);
\draw[arrow] (c2) -- (c4);
\draw[arrow] (c3) -- (c4);
\draw[arrow] (c4) to[bend right=25] (c6);
\draw[arrow] (c5) to[bend left=12] (c6);
\draw[arrow] (c6) -- (c8);
\draw[arrow] (c7) -- (c8);
\draw[arrow] (c8) -- (c9);
\draw[arrow] (c9) -- (c10);
\draw[arrow] (c10) -- (c11);

\node[below=0.2cm of c11, font=\scriptsize, align=center] (legend) {
    \textcolor{green!70!black}{\textbf{$\bullet$}} All TRUE / Retained\\
    \textcolor{gray}{\textbf{$\bullet$}} Not retained\\
    \textcolor{orange}{\textbf{$\bullet$}} Claim 8
};

\end{tikzpicture}
}
\caption{CF vs. DCF ADG Comparison (Correct Claims)}
\label{fig:case_study_retain}
\end{figure}
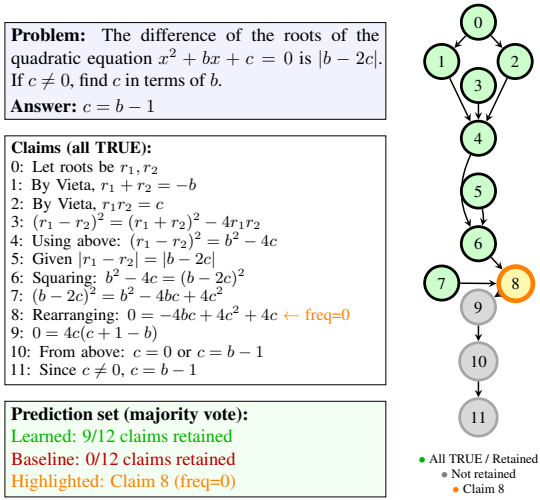

This problem's solution relies on Vieta's formulas---a valid but less common technique that the LLM rarely regenerates. As a result, key steps receive frequency $=0$, causing the baseline to reject the entire chain. DCF succeeds by combining complementary signals: when one feature is unreliable, others compensate.

\begin{table}[htbp]
\centering
\caption{DCF feature contributions for claim~8 (frequency=0).}
\label{tab:case_study_features}
\footnotesize
\setlength{\tabcolsep}{4pt}
\renewcommand{\arraystretch}{0.95}
\begin{tabular}{lrrr}
\toprule
Feature & Value & Weight & Contrib. \\
\midrule
\texttt{nx\_reachability} & 3.00 & 0.272 & +0.815 \\
\texttt{claim\_index} & 8.00 & 0.098 & +0.780 \\
\texttt{nx\_in\_degree} & 2.00 & 0.135 & +0.269 \\
\texttt{quadratic\_equations} & 1.00 & 0.229 & +0.229 \\
\texttt{nx\_out\_degree} & 1.00 & 0.176 & +0.176 \\
\texttt{problem\_relevance} & 1.00 & 0.144 & +0.144 \\
\texttt{coherent\_to\_ancestors} & 1.00 & 0.110 & +0.110 \\
\texttt{nx\_betweenness} & 0.22 & 0.292 & +0.064 \\
\midrule
\textbf{Total} & & & \textbf{2.66} \\
\bottomrule
\end{tabular}
\end{table}

Table~\ref{tab:case_study_features} shows how DCF retains Claim 8 despite frequency $=0$. Graph structure (\texttt{nx\_reachability}, \texttt{claim\_index}) provides strong positive signal, while connectivity and semantic features contribute additional evidence. The learned combination accumulates evidence across signals, enabling retention even when key features like frequency are unreliable. This flexibility is DCF's core advantage: rather than relying on one potentially unreliable signal, it learns which features to trust and how to weight them. See Appendix~\ref{sec:case_study_details} for the full methodology and an additional case study. \looseness=-1
\section{Related Work}\label{sec:related-work}

\textbf{Conformal Prediction.}
Conformal prediction (CP) \citep{vovk2005algorithmic, ShaferCP2008, AngelopoulosCPIntro2023} provides distribution-free coverage guarantees under exchangeability. Extensions include conformalized quantile regression \citep{RomanoConformalizedQuantileRegression2019} and the learn-then-test framework \citep{AngelopoulosLearnThenTest2025}, which calibrates predictive algorithms to achieve risk control---inspiring our treatment of conformal calibration as an optimization problem.\looseness=-1

Most relevant to our work, \citet{stutz2022learning} introduced ConfTr (Conformal Training), demonstrating that discrete quantile operations in CP can be relaxed into differentiable surrogates for end-to-end optimization---a key insight we extend to graph-structured reasoning. \citet{angelopoulos2024conformal} generalized this to conformal risk control with arbitrary monotone loss functions. \citet{GibbsAdaptiveCPDistributionShift2021} proposed adaptive conformal inference for distribution shift, though extending such adaptation to graph-structured dependencies remains open \cite{zhao2024conformalized}.\looseness=-1

\textbf{LLM Factuality.}
Recent work has extended conformal methods to LLMs for statistically rigorous uncertainty quantification \citep{zhou2025conformal}. \citet{kumar2023conformalpredictionlargelanguage} and \citet{quach2024conformallanguagemodeling} were among the first to adapt CP for LLM settings, while \citet{su2024apienoughconformalprediction} introduced sampling-based nonconformity scores for API-only access and \citet{wang2024conuconformaluncertaintylarge} incorporated self-consistency into CP for correctness coverage. These methods operate at the response or question level; a parallel line of work applies CP to structured, claim-level factuality.\looseness=-1

\citet{MohriHashimotoConformalFactuality2024} pioneered conformal factuality, decomposing LLM outputs into atomic claims and filtering based on calibrated risk thresholds. Their method provides high-probability correctness guarantees but treats claims independently without considering logical dependencies. \citet{rubin-toles2025conformal} addressed this with coherent factuality, extending CP to multi-step reasoning through approximate deducibility graphs. While their approach captures reasoning chain structure, it relies on hand-crafted scoring that leads to overly conservative filtering.\looseness=-1

\citet{CherianLLMValidity2024} developed enhanced conformal methods with conditional validity and differentiable scoring functions, demonstrating retention improvements but operating on independent claims without exploiting graph structure. \citet{min-etal-2023-factscore} introduced FActScore for fine-grained atomic evaluation, directly informing our claim-level analysis. Work on uncertainty detection \citep{kadavath2022languagemodelsmostlyknow, azaria-mitchell-2023-internal} suggests learned scoring functions could leverage model-internal features beyond surface-level frequency statistics. Complementary reasoning approaches \citep{WeiChainofThought2022, wang2023selfconsistency} improve reliability but lack formal statistical guarantees.\looseness=-1

Our work synthesizes these directions by making CF differentiable to enable learned scoring functions while preserving conformal guarantees.\looseness=-1

\section{Conclusion}\label{sec:conclusion}

We addressed the challenge of LLM factuality verification: existing conformal methods provide formal guarantees but sacrifice too many true claims to be practically useful. We introduced Differentiable Coherent Factuality (DCF), a framework that enables learning scoring functions for graph-structured conformal prediction while preserving theoretical guarantees.\looseness=-1

Our contributions span theory, methodology, and empirical validation. Theoretically, we proved that DCF recovers the original CF algorithm in the limit, with 93--100\% empirical agreement validating practical convergence. Methodologically, DCF provides the first differentiable relaxation of graph-structured LLM conformal reasoning, enabling end-to-end optimization through calibration and prediction. Empirically, learned scorers achieve up to 141\% retention improvements over frequency-based baselines, while deploying directly in the original CF algorithm at test time.\looseness=-1

More broadly, DCF bridges the gap between formal guarantees and usability. Statistical guarantees provide little value if outputs are too conservative to be informative; by substantially improving retention, DCF makes conformal factuality practical for real-world deployment.\looseness=-1


\section*{Impact Statement}

This work improves factual reliability of LLM outputs through formal statistical guarantees, contributing to safer AI deployment in accuracy-critical domains. We note that conformal guarantees are statistical (holding on average, not per-prediction). Additionally, as noted by \citet{CherianLLMValidity2024}, optimizing for marginal coverage may yield uneven reliability across subgroups or covariates, potentially misrepresenting performance on certain data subsets. Retained claims also still reflect the underlying LLM's knowledge limitations and biases.\looseness=-1

\section*{Limitations}
\label{sec:limitations}
While DCF achieves strong empirical results, we identify several limitations:

\paragraph{Quantile Instability at Strict Thresholds.}
At very low $\alpha$, DCF rejects all claims to maintain coverage: $\alpha \leq 0.02$ on MATH and $\alpha = 0.01$ on FELM. This reflects a fundamental limitation of conformal prediction with small calibration sets---quantile estimates become unstable when few samples fall below strict thresholds. Larger datasets would enable finer-grained calibration.\looseness=-1

\paragraph{Reduced Gains When Frequency Is Discriminative.}
DCF's advantage is largest when self-consistency frequency (or any baseline feature) poorly discriminates valid from invalid claims---as occurs with long reasoning chains, uncommon but valid techniques, or semantically incoherent outputs. On FELM, where simpler reasoning (4.0 vs.\ 7.3 claims per problem) makes frequency more reliable, DCF's gains are smaller and frequency-based baselines occasionally outperform DCF at mid-range $\alpha$ (Table~\ref{tab:felm_results}).\looseness=-1

\paragraph{MATH Dataset Scale and Annotation.}
MATH contains only 202 problems with manually annotated reasoning graphs, limiting robustness of evaluation. Additionally, our operationalization of coherent factuality treats fully hallucinated ADGs as incoherent, which influences graph metrics like \texttt{nx\_reachability}. While DCF outperforms \texttt{nx\_reachability} on MATH as a standalone baseline, the learned combination may not generalize to settings with well-structured but incorrect reasoning, where content-based features may become more critical.\looseness=-1

\paragraph{Scorer Architecture.}
Our linear scorer provides interpretability but limited capacity. Richer architectures---e.g., neural networks with hidden layers---could capture feature interactions that our linear scorer misses, providing enhanced claim retention.

\bibliography{example_paper}

@inproceedings{hendrycks2021measuring,
  author = {Hendrycks, Dan and Burns, Collin and Kadavath, Saurav and Arora, Akul and Basart, Steven and Tang, Eric and Song, Dawn and Steinhardt, Jacob},
  title = {Measuring Mathematical Problem Solving With the {MATH} Dataset},
  booktitle = {Proceedings of the Neural Information Processing Systems Track on Datasets and Benchmarks},
  year = {2021},
  volume = {1},
  editor = {Vanschoren, J. and Yeung, S.},
  url = {https://datasets-benchmarks-proceedings.neurips.cc/paper_files/paper/2021/file/be83ab3ecd0db773eb2dc1b0a17836a1-Paper-round2.pdf}
}

@inproceedings{zhao2024conformalized,
  title={Conformalized link prediction on graph neural networks},
  author={Zhao, Tianyi and Kang, Jian and Cheng, Lu},
  booktitle={Proceedings of the 30th ACM SIGKDD Conference on Knowledge Discovery and Data Mining},
  pages={4490--4499},
  year={2024}
}

@misc{kumar2023conformalpredictionlargelanguage,
      title={Conformal Prediction with Large Language Models for Multi-Choice Question Answering}, 
      author={Bhawesh Kumar and Charlie Lu and Gauri Gupta and Anil Palepu and David Bellamy and Ramesh Raskar and Andrew Beam},
      year={2023},
      eprint={2305.18404},
      archivePrefix={arXiv},
      primaryClass={cs.CL},
      url={https://arxiv.org/abs/2305.18404}, 
}

@misc{quach2024conformallanguagemodeling,
      title={Conformal Language Modeling}, 
      author={Victor Quach and Adam Fisch and Tal Schuster and Adam Yala and Jae Ho Sohn and Tommi S. Jaakkola and Regina Barzilay},
      year={2024},
      eprint={2306.10193},
      archivePrefix={arXiv},
      primaryClass={cs.CL},
      url={https://arxiv.org/abs/2306.10193}, 
}

@misc{su2024apienoughconformalprediction,
      title={API Is Enough: Conformal Prediction for Large Language Models Without Logit-Access}, 
      author={Jiayuan Su and Jing Luo and Hongwei Wang and Lu Cheng},
      year={2024},
      eprint={2403.01216},
      archivePrefix={arXiv},
      primaryClass={cs.CL},
      url={https://arxiv.org/abs/2403.01216}, 
}

@misc{wang2024conuconformaluncertaintylarge,
      title={ConU: Conformal Uncertainty in Large Language Models with Correctness Coverage Guarantees}, 
      author={Zhiyuan Wang and Jinhao Duan and Lu Cheng and Yue Zhang and Qingni Wang and Xiaoshuang Shi and Kaidi Xu and Hengtao Shen and Xiaofeng Zhu},
      year={2024},
      eprint={2407.00499},
      archivePrefix={arXiv},
      primaryClass={cs.CL},
      url={https://arxiv.org/abs/2407.00499}, 
}

@misc{openai_gpt5mini_2025,
  title        = {{GPT-5 Mini}},
  author       = {{OpenAI}},
  year         = {2025},
  howpublished = {\url{https://openai.com}},
  note         = {Large language model accessed via OpenAI API}
}

@book{vovk2005algorithmic,
  title={Algorithmic learning in a random world},
  author={Vovk, Vladimir and Gammerman, Alexander and Shafer, Glenn},
  year={2005},
  publisher={Springer}
}

@article{zhou2025conformal,
  title={Conformal prediction: A data perspective},
  author={Zhou, Xiaofan and Chen, Baiting and Gui, Yu and Cheng, Lu},
  journal={ACM Computing Surveys},
  year={2025},
  publisher={ACM New York, NY}
}

@inproceedings{
chen2023felm,
title={{FELM}: Benchmarking Factuality Evaluation of Large Language Models},
author={Shiqi Chen and Yiran Zhao and Jinghan Zhang and I-Chun Chern and Siyang Gao and Pengfei Liu and Junxian He},
booktitle={Thirty-seventh Conference on Neural Information Processing Systems Datasets and Benchmarks Track},
year={2023},
url={https://openreview.net/forum?id=jSO7Vgolc6}
}

@inproceedings{
grover2018stochastic,
title={Stochastic Optimization of Sorting Networks via Continuous Relaxations},
author={Aditya Grover and Eric Wang and Aaron Zweig and Stefano Ermon},
booktitle={International Conference on Learning Representations},
year={2019},
url={https://openreview.net/forum?id=H1eSS3CcKX},
}

@article{ShaferCP2008,
author = {Shafer, Glenn and Vovk, Vladimir},
title = {A Tutorial on Conformal Prediction},
year = {2008},
issue_date = {6/1/2008},
publisher = {JMLR.org},
volume = {9},
issn = {1532-4435},
journal = {J. Mach. Learn. Res.},
month = jun,
pages = {371–421},
numpages = {51}
}

@article{AngelopoulosCPIntro2023,
url = {http://dx.doi.org/10.1561/2200000101},
year = {2023},
volume = {16},
journal = {Foundations and Trends® in Machine Learning},
title = {Conformal Prediction: A Gentle Introduction},
doi = {10.1561/2200000101},
issn = {1935-8237},
number = {4},
pages = {494-591},
author = {Anastasios N. Angelopoulos and Stephen Bates}
}

@inproceedings{RomanoConformalizedQuantileRegression2019,
 author = {Romano, Yaniv and Patterson, Evan and Candes, Emmanuel},
 booktitle = {Advances in Neural Information Processing Systems},
 editor = {H. Wallach and H. Larochelle and A. Beygelzimer and F. d\textquotesingle Alch\'{e}-Buc and E. Fox and R. Garnett},
 pages = {},
 publisher = {Curran Associates, Inc.},
 title = {Conformalized Quantile Regression},
 url = {https://proceedings.neurips.cc/paper_files/paper/2019/file/5103c3584b063c431bd1268e9b5e76fb-Paper.pdf},
 volume = {32},
 year = {2019}
}

@inproceedings{MohriHashimotoConformalFactuality2024,
author = {Mohri, Christopher and Hashimoto, Tatsunori},
title = {Language models with conformal factuality guarantees},
year = {2024},
publisher = {JMLR.org},
booktitle = {Proceedings of the 41st International Conference on Machine Learning},
articleno = {1468},
numpages = {19},
location = {Vienna, Austria},
series = {ICML'24}
}

@inproceedings{
rubin-toles2025conformal,
title={Conformal Language Model Reasoning with Coherent Factuality},
author={Maxon Rubin-Toles and Maya Gambhir and Keshav Ramji and Aaron Roth and Surbhi Goel},
booktitle={The Thirteenth International Conference on Learning Representations},
year={2025},
url={https://openreview.net/forum?id=AJpUZd8Clb}
}

@inproceedings{
angelopoulos2024conformal,
title={Conformal Risk Control},
author={Anastasios Nikolas Angelopoulos and Stephen Bates and Adam Fisch and Lihua Lei and Tal Schuster},
booktitle={The Twelfth International Conference on Learning Representations},
year={2024},
url={https://openreview.net/forum?id=33XGfHLtZg}
}

@inproceedings{CherianLLMValidity2024,
author = {Cherian, John J. and Gibbs, Isaac and Cand\`{e}s, Emmanuel J.},
title = {Large language model validity via enhanced conformal prediction methods},
year = {2024},
isbn = {9798331314385},
publisher = {Curran Associates Inc.},
address = {Red Hook, NY, USA},
booktitle = {Proceedings of the 38th International Conference on Neural Information Processing Systems},
articleno = {3645},
numpages = {31},
location = {Vancouver, BC, Canada},
series = {NIPS '24}
}

@inproceedings{
stutz2022learning,
title={Learning Optimal Conformal Classifiers},
author={David Stutz and Krishnamurthy Dj Dvijotham and Ali Taylan Cemgil and Arnaud Doucet},
booktitle={International Conference on Learning Representations},
year={2022},
url={https://openreview.net/forum?id=t8O-4LKFVx}
}

@article{AngelopoulosLearnThenTest2025,
author = {Anastasios N. Angelopoulos and Stephen Bates and Emmanuel J. Cand{\`e}s and Michael I. Jordan and Lihua Lei},
title = {{Learn then test: Calibrating predictive algorithms to achieve risk control}},
volume = {19},
journal = {The Annals of Applied Statistics},
number = {2},
publisher = {Institute of Mathematical Statistics},
pages = {1641 -- 1662},
year = {2025},
doi = {10.1214/24-AOAS1998},
URL = {https://doi.org/10.1214/24-AOAS1998}
}

@inproceedings{GibbsAdaptiveCPDistributionShift2021,
 author = {Gibbs, Isaac and Candes, Emmanuel},
 booktitle = {Advances in Neural Information Processing Systems},
 editor = {M. Ranzato and A. Beygelzimer and Y. Dauphin and P.S. Liang and J. Wortman Vaughan},
 pages = {1660--1672},
 publisher = {Curran Associates, Inc.},
 title = {Adaptive Conformal Inference Under Distribution Shift},
 url = {https://proceedings.neurips.cc/paper_files/paper/2021/file/0d441de75945e5acbc865406fc9a2559-Paper.pdf},
 volume = {34},
 year = {2021}
}

@inproceedings{min-etal-2023-factscore,
    title = "{FA}ct{S}core: Fine-grained Atomic Evaluation of Factual Precision in Long Form Text Generation",
    author = "Min, Sewon  and
      Krishna, Kalpesh  and
      Lyu, Xinxi  and
      Lewis, Mike  and
      Yih, Wen-tau  and
      Koh, Pang  and
      Iyyer, Mohit  and
      Zettlemoyer, Luke  and
      Hajishirzi, Hannaneh",
    editor = "Bouamor, Houda  and
      Pino, Juan  and
      Bali, Kalika",
    booktitle = "Proceedings of the 2023 Conference on Empirical Methods in Natural Language Processing",
    month = dec,
    year = "2023",
    address = "Singapore",
    publisher = "Association for Computational Linguistics",
    url = "https://aclanthology.org/2023.emnlp-main.741/",
    doi = "10.18653/v1/2023.emnlp-main.741",
    pages = "12076--12100"
}

@misc{kadavath2022languagemodelsmostlyknow,
      title={Language Models (Mostly) Know What They Know}, 
      author={Saurav Kadavath and Tom Conerly and Amanda Askell and Tom Henighan and Dawn Drain and Ethan Perez and Nicholas Schiefer and Zac Hatfield-Dodds and Nova DasSarma and Eli Tran-Johnson and Scott Johnston and Sheer El-Showk and Andy Jones and Nelson Elhage and Tristan Hume and Anna Chen and Yuntao Bai and Sam Bowman and Stanislav Fort and Deep Ganguli and Danny Hernandez and Josh Jacobson and Jackson Kernion and Shauna Kravec and Liane Lovitt and Kamal Ndousse and Catherine Olsson and Sam Ringer and Dario Amodei and Tom Brown and Jack Clark and Nicholas Joseph and Ben Mann and Sam McCandlish and Chris Olah and Jared Kaplan},
      year={2022},
      eprint={2207.05221},
      archivePrefix={arXiv},
      primaryClass={cs.CL},
      url={https://arxiv.org/abs/2207.05221}, 
}

@inproceedings{azaria-mitchell-2023-internal,
    title = "The Internal State of an {LLM} Knows When It{'}s Lying",
    author = "Azaria, Amos  and
      Mitchell, Tom",
    editor = "Bouamor, Houda  and
      Pino, Juan  and
      Bali, Kalika",
    booktitle = "Findings of the Association for Computational Linguistics: EMNLP 2023",
    month = dec,
    year = "2023",
    address = "Singapore",
    publisher = "Association for Computational Linguistics",
    url = "https://aclanthology.org/2023.findings-emnlp.68/",
    doi = "10.18653/v1/2023.findings-emnlp.68",
    pages = "967--976"
}

@inproceedings{WeiChainofThought2022,
author = {Wei, Jason and Wang, Xuezhi and Schuurmans, Dale and Bosma, Maarten and Ichter, Brian and Xia, Fei and Chi, Ed H. and Le, Quoc V. and Zhou, Denny},
title = {Chain-of-thought prompting elicits reasoning in large language models},
year = {2022},
isbn = {9781713871088},
publisher = {Curran Associates Inc.},
address = {Red Hook, NY, USA},
booktitle = {Proceedings of the 36th International Conference on Neural Information Processing Systems},
articleno = {1800},
numpages = {14},
location = {New Orleans, LA, USA},
series = {NIPS '22}
}

@inproceedings{
wang2023selfconsistency,
title={Self-Consistency Improves Chain of Thought Reasoning in Language Models},
author={Xuezhi Wang and Jason Wei and Dale Schuurmans and Quoc V Le and Ed H. Chi and Sharan Narang and Aakanksha Chowdhery and Denny Zhou},
booktitle={The Eleventh International Conference on Learning Representations },
year={2023},
url={https://openreview.net/forum?id=1PL1NIMMrw}
}
\bibliographystyle{icml2026}

\newpage
\appendix
\onecolumn

\section{Proofs}
\label{sec:proofs}

\subsection{Statement and Proof of Theorem \ref{thm:calibration}}\label{sec:proof-calibration}
Fix ADG $\mathcal{G} = (V,E)$, false claim nodes $V^- \subseteq V$, scorer $\pi_\theta$, threshold grid $\mathcal{T}$, and graph family $\mathcal{U_T}$. Let $p_{v,\tau} = \sigma\left((\tau - r_v + \sqrt{T_p})/T_p\right)$ with\footnote{In practice, we let $\lambda$ vary as a hyperparameter. The theorem's condition could be guaranteed by making $\lambda$ differentiably approximate $1/(\sup(\mathcal{T}) - \inf(\mathcal{T}))$, then applying a limit to recover the hard nonconformity score.} $\sup(\lambda\mathcal{T})-\inf(\lambda\mathcal{T}) \leq 1$. Then:
\begin{equation} \label{eqn:calibration_convergence}
\lim_{\beta\rightarrow\infty}\lim_{\tau_s\rightarrow0}\lim_{T_p\rightarrow0^+}\tilde{\tau} = \nu(X, Y, \mathcal{U}_\mathcal{T})
\end{equation}
where $\nu(X, Y, \mathcal{U}_\mathcal{T})$ is the hard nonconformity score from CF.\footnote{The theorem requires a $\sqrt{T_p}$ margin in \eqref{eqn:softkeep}, but in practice a margin of 0 yields similar soft keep probabilities.} Setting $\tau_s = T_p^{s}$ and $\beta = T_p^a$ for $s\in(0,1), a<0$ yields single-limit convergence.

\begin{proof} 

For simplicity of notation set $T = T_p$. $\tau = r_{v_j}$ for some ${v_j}$. Checking limits, note as $T \rightarrow 0^+$, $p_{v,\tau} \longrightarrow \Theta(\tau - r_v)$ pointwise, where $\Theta(x)$ is the Heaviside step function with $\Theta(0) = 1$. 
\begin{case}
    $r_v > \tau$ and vertex $v$ is not coherently factual. Then 
    \begin{equation}
    r_v > \tau \Rightarrow \lim_{T \rightarrow 0+} p_{v,\tau} = 0. 
    \end{equation}
\end{case}
\begin{case}
    $r_v \leq \tau$ and vertex $v$ is  coherently factual. Then 
    \begin{equation}
    r_v \leq \tau \Rightarrow \lim_{T \rightarrow 0+} p_{v,\tau} = 1. 
    \end{equation}
\end{case}
After taking an ancestral geometric mean, $q_{v,\tau} = (\Pi_{u \in \text{Anc}(v) \cup \{v\}} \, p_{u,\tau}^{w_u})^{1/(|\text{Anc}(v)| + 1)}$ is sent to 0 if any of $v$ or $v$’s ancestors had risk score greater than $\tau$, and is sent to 1 otherwise. Then 
\begin{equation} \label{eqn: Q_tau limit}
    \lim_{T \rightarrow 0^+} Q_\tau =  \left\{ \begin{array}{rcl}
    0 & \mbox{for} & \exists v\in V^- \,|\, \forall u \in \text{Anc}(v) \cup \{v\}, r_{u} \leq \tau \\ 
    1 & \mbox{for} & \forall v\in V^-, \exists u \in \text{Anc}(v) \cup \{v\} \,|\, r_{u} > \tau
    \end{array}\right.,
\end{equation}
and $\tau_s$ has no influence. Then violation has that
\begin{equation}
    \lim_{T \rightarrow 0^+} V_\tau =  \left\{ \begin{array}{rcl}
    1 & \mbox{for} & U_\tau \text{ is not Coherently Factual} \\ 
    0 & \mbox{for} & U_\tau \text{ is Coherently Factual}
    \end{array}\right. .
\end{equation}
In the limit $\beta \rightarrow \infty$, $w_{\tau_i}^\text{cal} = \indicator \{s_{\tau_i} = \sup\{s_{\tau_{j}} \,|\, \tau_j \in \mathcal{T}\}\}$. Because $\sup(\lambda\mathcal{T})-\inf(\lambda\mathcal{T}) \leq 1$, if any subgraph $U_{\tau_i}$ is not CF, then $V_{\tau_i} = 1$, so the maximum $s_{\tau_i}$ corresponds to the greatest $\lambda\tau_{i}$ such that $V_{\tau_i} = 0$ and $U_{\tau_i}$ is CF. We lastly need the fact that
\begin{equation}
\forall \tau'\in \mathcal{T} \,|\, \tau' \leq \tau, U_\tau \text{ Coherently Factual} \Rightarrow U_{\tau'} \text{ Coherently Factual},
\end{equation}
which holds since $\tau' \leq \tau$ and $Q_\tau = 1$ implies $Q_{\tau'}$ = 1. Therefore
\begin{equation}
\lim_{\beta\rightarrow\infty}\lim_{\tau_s\rightarrow 0^+}\lim_{T\rightarrow 0^+}\tilde{\tau} = \nu(X, Y, \mathcal{U}_\mathcal{T})
\end{equation}

It remains to show that setting $\tau_s = T^{s}$ and $\beta = T^a$ for $s\in(0,1), a<0$ yields single-limit convergence. Note that $\beta(T)\rightarrow\infty$ and $\tau_s \rightarrow 0$ as $T \rightarrow 0$, let $s_{\text{max}}(T) = \argmax_{s\in\{s_\tau\}_{\tau\in\mathcal{T}}} \lim_{T \rightarrow 0^+} s_\tau$, and let $n_\text{max}$ be the number of $s_\tau$ with the same limit as $s_{\text{max}}$. Then for $\tau_\text{max}$ corresponding to $s_\text{max}$,
\begin{align*}
    \lim_{T \rightarrow 0^+} \tilde{\tau} = \lim_{T \rightarrow 0^+} \frac{\sum_{\tau \in \mathcal{T}} \exp(\beta s_\tau) \tau}{\sum_{\tau \in \mathcal{T}} \exp(\beta s_\tau)} = \lim_{T \rightarrow 0^+} \frac{\sum_{\tau \in \mathcal{T}} \exp(\beta s_\tau - \beta s_\text{max}) \tau}{\sum_{\tau \in \mathcal{T}} \exp(\beta s_\tau - \beta s_\text{max})} = \frac{\tau_\text{max} n_\text{max}}{n_\text{max}} = \tau_\text{max}, 
\end{align*}
if $s_\text{max}$ is well defined, which is seen through convergence of $Q_\tau^{1/\tau_s}$. Now it remains to show that 
\begin{equation} \label{eqn: Q_tau^exp}
    \lim_{T \rightarrow 0^+} Q_\tau^{1/\tau_s} =  \left\{ \begin{array}{rcl}
    0 & \mbox{for} & \exists v\in V^- \,|\, \forall u \in \text{Anc}(v) \cup \{v\}, r_{u} \leq \tau \\ 
    1 & \mbox{for} & \forall v\in V^-, \exists u \in \text{Anc}(v) \cup \{v\} \,|\, r_{u} > \tau
    \end{array}\right.,
\end{equation}
similar to equation \ref{eqn: Q_tau limit} but with an exponent that could prevent convergence to 1. We first show that in the case where $Q_\tau \rightarrow 1$, we also have $Q_\tau^{1/\tau_s} \rightarrow 1$. So $\forall v\in V^-, \exists u \in \text{Anc}(v) \cup \{v\} \,|\, r_{u} > \tau$. We can write
\begin{align*}
    \lim_{T \rightarrow 0^+} Q_\tau^{|V^-|/\tau_s} = \Pi_{v\in V^-} \lim_{T\rightarrow 0^+} (1-q_{v,\tau})^{1/\tau_s} .
\end{align*}
Set $L_v = \lim_{T \rightarrow 0^+} (1-q_{v,\tau})^{1/\tau_s}$. We want $L_v \rightarrow 1$ for all $v \in V^-$. Notice we can express $L_v$ in the following manner:

\begin{equation} \label{eqn: L_v}
    L_v = \lim_{T \rightarrow 0^+} (1-f_1(T)f_2(T)...)^{1/\tau_s},
\end{equation}
where each $f_i=p_{u,\tau}^{c_u}$ with appropriate constant $c_u>0$. At least one $f_i$ has $\lim_{T\rightarrow0^+} f_i(T) = 0$ and  others approach 1 or 0. Now
\begin{align*}
    L_v = \exp \lim_{T \rightarrow 0^+} \ln (1-f_1(T)f_2(T)...)/\tau_s = \exp \lim_{T \rightarrow 0^+} \frac{-1}{1-f_1(T)f_2(T)...} (f_1(T)f_2(T)...)(\sum_i f_i(T)'/f_i(T))(\tau_s')^{-1}
\end{align*}
by $\text{L'H}\hat{\text{o}}\text{pital's}$ rule. Now $\frac{-1}{1-f_1(T)f_2(T)...} \rightarrow -1$, and $\tau_s' = \frac1s T^{s-1} \rightarrow \infty$ so $(\tau_s')^{-1}\rightarrow 0$. The remaining terms are:
\begin{align*}
    (f_1(T)f_2(T)...)\sum_i f_i(T)'/f_i(T) = \sum_i f_1(T)f_2(T)...\hat{f}_i(T)...f_n(T)f_i(T)'
\end{align*}
where the hat indicates omission of that term. Each $f_i$ approaches 1 or 0, so we must check $f_i(T)'$ for convergence. Explicitly, for an $f_i$ there is some $u\in \{v\}\cup \text{Anc}(v)$ such that 
\begin{align*}
f_i(T) = p_{u,\tau}^{w_u/(1+\gamma|\text{Anc}(v)|)} = \frac{1}{(1+\exp((r_u - \tau - \sqrt{T})/T))^{c_2}} = \frac{1}{(1+\exp(c_1/T - 1/\sqrt{T}))^{c_2}} ,
\end{align*}
with $c_2 = \frac{w_u}{1+\gamma|\text{Anc}(v)|} \in (0,1)$ and $c_1 = r_u-\tau$. After differentiation and taking the limit, we find
\begin{align*}
    \lim_{T\rightarrow 0^+} f_i(T)' = 0. 
\end{align*}
Thus $L_v = \exp(0) = 1$ for each $v\in V^-$ when $\forall v\in V^-, \exists u \in \text{Anc}(v) \cup \{v\} \,|\, r_{u} > \tau$. Then $\lim_{T\rightarrow 0^+} Q_\tau^{1/\tau_s} = 1$. 

The remaining case to prove is when $\exists v\in V^- \,|\, \forall u \in \text{Anc}(v) \cup \{v\}, r_{u} \leq \tau$. Here, there exists a $v\in V^-$ such that each $f_i$ in equation \ref{eqn: L_v} approaches 1, and thus $L_v$ approaches 0 and $\lim_{T\rightarrow 0^+} Q_\tau^{1/\tau_s} = 0$. This completes the proof of equation \ref{eqn: Q_tau^exp} and thus that setting $\tau_s = T^{s}$ and $\beta = T^a$ for $s\in(0,1), a<0$ yields
\begin{equation}
    \lim_{T\rightarrow 0^+}\tilde{\tau} = \nu(X, Y, \mathcal{U}_\mathcal{T}).
\end{equation}

\end{proof}

\subsection{Statement and Proof of Theorem \ref{thm:prediction}}\label{sec:proof-prediction}
Fix graph $\mathcal{G} = (V, E)$, risk scores $\{r_v\}_{v\in V}$, and calibrated threshold $\tau_\alpha$. Suppose $\{\tau\in\mathcal{T} \,|\, \tau<\tau_\alpha\}\neq\emptyset$ and $\forall u\in V, w_u \neq 0$. Let\footnote{This requires a $\sqrt{\tau_z}$ margin in equation \ref{eqn: w_tau unnorm}. In practice, this was left at zero.} $w_\tau^{\text{(unnorm)}} = \exp(\beta \cdot \tau) \cdot \sigma\left(\frac{\tau_\alpha - \tau - \sqrt{\tau_z}}{\tau_z}\right)$. Then:
\begin{equation}
\label{eqn:prediction_convergence}
     \{v\in V \mid \lim_{T_p\to0^+}\lim_{\beta\to\infty}\lim_{\tau_z\to0^+} q_v \in [0.5, 1]\} = U_{\text{filtered}},
\end{equation}
where $U_{\text{filtered}}$ is the CF prediction.\footnote{This theorem uses equation \eqref{eqn:softkeep} in its original form; extending it to support the margin from Theorem \ref{thm:calibration} is straightforward.} Setting $\tau_z = T_p^{ab}$, $\beta = 1/T_p^a$ for $a>0$, $b>2$ yields single-limit convergence. 

\begin{proof} For simplicity of notation let $T=T_p$ and $\hat{\tau}_\alpha = \tau_\alpha$. Since
\begin{equation}
    \lim_{T\rightarrow 0^+} p_{v,\tau} = \left\{ \begin{array}{rcl}
    1 & \mbox{for} & r_v < \tau \\ 
    \frac12 & \mbox{for} & r_v = \tau \\
    0 & \mbox{for} & r_v > \tau
    \end{array}\right. ,
\end{equation}
we have 
\begin{equation} \label{eqn: q_v,tau limit}
    \lim_{T\rightarrow 0^+} q_{v,\tau} \in \left\{ \begin{array}{rcl}
    \{0\} & \mbox{if} & {\exists u\in \text{Anc}(v)\cup\{v\} \,:\, r_u > \tau} \\ 
    {[\frac12, 1]} & \mbox{if} & {\forall  u\in \text{Anc}(v)\cup\{v\},\, r_u \leq \tau}
    \end{array}\right. .
\end{equation}
Note 
\begin{equation}
    \lim_{\tau_z \rightarrow 0^+} w_\tau^{(\text{unnorm})} = \left\{ \begin{array}{rcl}
    0 & \mbox{for} & {\tau\geq\tau_\alpha} \\  
    \exp(\beta\tau) & \mbox{for} & {\tau<\tau_\alpha}
    \end{array}\right. .
\end{equation}
Now as $\beta\rightarrow\infty$, $w_\tau \rightarrow \indicator \{\tau=\sup\{\tau\in\mathcal{T} \,|\, \tau< \tau_\alpha\}\}/n_j$, where $n_j = |\{\tau\in\mathcal{T} \,|\, \tau=\sup\{\tau\in\mathcal{T} \,|\, \tau<\tau_\alpha\}\}|$. Since identical $\tau$ give identical coefficients $q_{v,\tau}$ for each $w_\tau$ in equation \ref{eqn: w_tau}, we find that
\begin{equation} \label{eqn: q_v limit}
    \lim_{\beta\rightarrow \infty} \lim_{\tau_z\rightarrow 0^+} q_{v} = q_{v,\tau} \,:\, \tau = \sup\{\tau\in\mathcal{T} \,|\, \tau< \tau_\alpha\}
\end{equation}
Then $\{v\in V \,|\, \lim_{T\rightarrow0^+}\lim_{\beta\rightarrow\infty}\lim_{\tau_z\rightarrow0^+} q_v \in[0.5, 1] \} = \{v\in V \,|\, \forall u\in\text{Anc}(v)\cup\{v\},\, r_u \leq \sup\{\tau\in\mathcal{T} \,|\, \tau< \tau_\alpha\}\}$, which is equivalent to $U_\text{filtered}$. 

It remains to show that setting $\beta=1/T^a$ and $\tau_z = T^{ab}$ for $a> 0, b>1$ yields single-limit convergence. We may write
\begin{equation}
    q_v = \frac{\sum_\tau \exp(\tau/T^a)\sigma(\frac{\tau_\alpha - \tau - T^{ab/2}}{T^{ab}}) q_{v,\tau}}{\sum_\tau \exp(\tau/T^a)\sigma(\frac{\tau_\alpha - \tau-T^{ab/2}}{T^{ab}})}. 
\end{equation}
Let $\tau_j = \sup\{\tau\in\mathcal{T} \,|\, \tau<\tau_\alpha\}$ so 
\begin{equation}
    q_v = \frac{\sum_\tau \exp(\tau/T^a - \tau_j/T^a)\sigma(\frac{\tau_\alpha - \tau-T^{ab/2}}{T^{ab}}) q_{v,\tau}}{\sum_\tau \exp(\tau/T^a - \tau_j/T^a)\sigma(\frac{\tau_\alpha - \tau-T^{ab/2}}{T^{ab}})}. 
\end{equation}
Only terms matching $\tau_j$ survive in both sums. Terms with $\tau<\tau_j$ approach 0 via the first exponential, and terms with $\tau>\tau_j$ also have $\tau\geq\tau_\alpha$, so
\begin{equation}
    \lim_{T\rightarrow 0^+} \exp(\frac{\tau - \tau_j}{T^a})\sigma(\frac{\tau_\alpha-\tau-T^{ab/2}}{T^{ab}}) = \lim_{T\rightarrow 0^+} \frac1{e^{(\tau_j-\tau)/T^a} + \exp(\frac{(\tau-\tau_\alpha) + (\tau_j-\tau)T^b + T^{ab/2}}{T^{ab}})} = 0
\end{equation}
since $b>2$ and $a>0$. Then
\begin{equation}
    \lim_{T\rightarrow 0^+} q_v = \frac{\sum_\tau \indicator\{\tau = \tau_j\} \lim_{T\rightarrow 0^+} q_{v,\tau}}{\sum_\tau \indicator\{\tau = \tau_j\}} = \lim_{T\rightarrow 0^+}q_{v,\tau_j}. 
\end{equation}
Since equation ~\ref{eqn: q_v,tau limit} holds in this case and $\tau_j$ is identical to the $\tau$ of equation \ref{eqn: q_v limit}, we have
\begin{equation}
    \{v\in V \,|\, \lim_{T\rightarrow0^+} q_v \in[0.5, 1] \} = U_\text{filtered}
\end{equation}

\end{proof}

\section{Algorithms}
\label{sec:algorithms}

\subsection{Differentiable Calibration}\label{sec:alg-calibration}
\begin{algorithm}[H]
\caption{Differentiable Calibration}
\label{alg:calibration}
\begin{algorithmic}[1]
\STATE \textbf{Input:} Calibration set $\mathcal{D}_{\text{cal}} = \{(\mathcal{G}_i, \{x_v^{(i)}\}, \{y_v^{(i)}\})\}_{i=1}^{n}$, scorer $\pi_\theta$, target miscoverage $\alpha$, hyperparameters $C, T_p, \beta, \gamma, \lambda, \tau_s, \rho, \epsilon, m \in \mathbb{R}$
\STATE \textbf{Output:} Calibrated threshold $\hat{\tau}_\alpha$

\FOR{each graph $(\mathcal{G}, \{x_v\}, \{y_v\})$ in $\mathcal{D}_{\text{cal}}$}
    \STATE \textbf{Step 1: Scoring and risk.}
    \FOR{$v \in V$}
        \STATE \quad $r_v = C - \pi_\theta(x_v)$
    \ENDFOR

    \STATE \textbf{Step 2: Build $\tau$-grid.}
    \STATE $\mathcal{T} = \{\min_v r_v - m, \; \{r_v\}_{v\in V}, \; \max_v r_v + m\}$ (sorted)

    \STATE \textbf{Step 3: Soft keep.}
    \FOR{$(v,\tau) \in V \times \mathcal{T}$}
        \STATE $p_{v,\tau} = \sigma\left(\frac{\tau - r_v}{T_p}\right)$
    \ENDFOR

    \STATE \textbf{Step 4: Ancestor coherence.}
    \FOR{$(v,\tau) \in V \times \mathcal{T}$}
        \STATE $w_u = \begin{cases} \gamma & \text{if } u \in \text{Anc}(v) \\ 1 & \text{if } u = v \end{cases}$ for all $u \in \text{Anc}(v) \cup \{v\}$
        \STATE $\log q_{v,\tau} = \frac{\sum_{u \in \text{Anc}(v)\cup\{v\}} w_u \cdot \log(p_{u,\tau} + \epsilon)}{\sum_{u \in \text{Anc}(v)\cup\{v\}} w_u}$
    \ENDFOR

    \STATE \textbf{Step 5: Validity on negatives.}
    \STATE $V^- = \{v \in V : y_v = 0\}$
    \FOR{$\tau \in \mathcal{T}$}
        \STATE $\log Q_{\tau} = \frac{1}{|V^-|}\sum_{v\in V^-} \log(1 - \exp(\log q_{v,\tau}) + \epsilon)$
    \ENDFOR

    \STATE \textbf{Step 6: Violation and soft supremum.}
    \FOR{$\tau \in \mathcal{T}$}
        \STATE $V_{\tau} = 1 - \exp(\log (Q_{\tau}) / \tau_s)$
    \ENDFOR
    \STATE Normalize: $\hat{\tau}_{\text{norm}} = \text{MinMax}(\mathcal{T})$, $\hat{V} = \text{MinMax}(\{V_\tau\})$ \hfill \COMMENT{Scale to $[0,1]$}
    \FOR{$\tau \in \mathcal{T}$}
        \STATE $s_{\tau} = \hat{\tau}_{\text{norm}} - \lambda \cdot \hat{V}_{\tau}$
        \STATE $w_{\tau}^\text{cal} = \frac{\exp(\beta s_{\tau})}{\sum_{\tau' \in \mathcal{T}} \exp(\beta s_{\tau'})}$
    \ENDFOR
    \STATE $\tilde{\tau} = \sum_{\tau \in \mathcal{T}} w_{\tau}^\text{cal} \cdot \tau$
    \STATE Store $\tilde{\tau}$
\ENDFOR

\STATE \textbf{Calibration via soft quantile:}
\STATE $q = \frac{\lceil (n+1)(1-\alpha) \rceil}{n}$
\STATE $\hat{\tau}_\alpha = \text{SoftQuantile}(\{\tilde{\tau}_1, \ldots, \tilde{\tau}_n\}, q, \rho)$

\STATE \textbf{return} $\hat{\tau}_\alpha$
\end{algorithmic}
\end{algorithm}

\subsection{Differentiable Prediction}\label{sec:alg-prediction}
\begin{algorithm}[H]
\caption{Differentiable Prediction}
\label{alg:prediction}
\begin{algorithmic}[1]
\STATE \textbf{Input:} Test set $\mathcal{D}_{\text{test}} = \{(\mathcal{G}_i, \{x_v^{(i)}\})\}_{i=1}^{M}$, scorer $\pi_\theta$, calibrated threshold $\hat{\tau}_\alpha$, hyperparameters $C, T_p, \beta, \gamma, \tau_z, \epsilon, m \in \mathbb{R}$
\STATE \textbf{Output:} Soft retention probabilities $\{\mathbf{q}^{(i)}\}_{i=1}^{M}$ where $\mathbf{q}^{(i)} \in [0,1]^{|V_i|}$

\FOR{each graph $(\mathcal{G}, \{x_v\})$ in $\mathcal{D}_{\text{test}}$}
    \STATE \textbf{Step 1: Scoring and risk.}
    \FOR{$v \in V$}
        \STATE \quad $r_v = C - \pi_\theta(x_v)$
    \ENDFOR

    \STATE \textbf{Step 2: Build $\tau$-grid.}
    \STATE $\mathcal{T} = \{\min_v r_v - m, \; \{r_v\}_{v\in V}, \; \max_v r_v + m\}$ (sorted)

    \STATE \textbf{Step 3: Soft keep.}
    \FOR{$(v,\tau) \in V \times \mathcal{T}$}
        \STATE $p_{v,\tau} = \sigma\left(\frac{\tau - r_v}{T_p}\right)$
    \ENDFOR

    \STATE \textbf{Step 4: Ancestor coherence.}
    \FOR{$(v,\tau) \in V \times \mathcal{T}$}
        \STATE $w_u = \begin{cases} \gamma & \text{if } u \in \text{Anc}(v) \\ 1 & \text{if } u = v \end{cases}$ for all $u \in \text{Anc}(v) \cup \{v\}$
        \STATE $\log q_{v,\tau} = \frac{\sum_{u \in \text{Anc}(v)\cup\{v\}} w_u \cdot \log(p_{u,\tau} + \epsilon)}{\sum_{u \in \text{Anc}(v)\cup\{v\}} w_u}$
    \ENDFOR

    \STATE \textbf{Step 5: Soft gated argmax.}
    \FOR{$\tau \in \mathcal{T}$}
        \STATE $w_\tau^{(\text{unnorm})} = \exp(\beta \cdot \tau) \cdot \sigma\left(\frac{\hat{\tau}_\alpha - \tau}{\tau_z}\right)$
    \ENDFOR
    \STATE Normalize: $w_\tau = \frac{w_\tau^{(\text{unnorm)}}}{\sum_{\tau' \in \mathcal{T}} w_{\tau'}^{(\text{unnorm})}}$

    \STATE \textbf{Step 6: Weighted combination.}
    \FOR{$v \in V$}
        \STATE $q_v = \sum_{\tau \in \mathcal{T}} w_\tau \cdot q_{v,\tau}$
    \ENDFOR
    \STATE Store $\mathbf{q} = (q_v)_{v \in V}$
\ENDFOR

\STATE \textbf{return} $\{\mathbf{q}^{(1)}, \ldots, \mathbf{q}^{(M)}\}$
\end{algorithmic}
\end{algorithm}

\subsection{Training Procedure}\label{sec:alg-training}
\begin{algorithm}[H]
\caption{Training $\pi_{\theta}$}
\label{alg:training}
\begin{algorithmic}[1]
\STATE \textbf{Input:} Dataset $\mathcal{D} = \{(\mathcal{G}_i, \{x_v^{(i)}\}, \{y_v^{(i)}\})\}_{i=1}^{N}$, scorer $\pi_\theta$, miscoverage $\alpha$, learning rate $\eta$, epochs $E$, hyperparameters $C, T_p, \beta, \gamma, \lambda, \tau_s, \rho, \epsilon, m$
\STATE \textbf{Output:} Trained scorer $\pi_\theta$

\STATE \textbf{Step 1: Data split.}
\STATE Partition $\mathcal{D}$ into train $\mathcal{D}_{\text{train}}$ and validation $\mathcal{D}_{\text{val}}$ sets

\STATE \textbf{Step 2: Initialize.}
\STATE Initialize scorer parameters $\theta$ and Adam optimizer with learning rate $\eta$

\FOR{epoch $e = 1, \ldots, E$}
    \STATE \textbf{Step 3: Training.}
    \STATE Split $\mathcal{D}_{\text{train}}$ into calibration $\mathcal{D}_{\text{cal}}$ and prediction $\mathcal{D}_{\text{pred}}$ subsets
    \STATE Calibrate: $\hat{\tau}_\alpha \gets \text{DifferentiableCalibration}(\mathcal{D}_{\text{cal}}, \pi_\theta, \alpha, \ldots)$ (Algorithm~\ref{alg:calibration})
    \STATE Predict: $\{\mathbf{p}^{(i)}\} \gets \text{DifferentiablePrediction}(\mathcal{D}_{\text{pred}}, \pi_\theta, \hat{\tau}_\alpha, \ldots)$ (Algorithm~\ref{alg:prediction})
    \STATE Compute loss: $\mathcal{L} = -\frac{1}{|\mathcal{D}_{\text{pred}}|} \sum_{i \in \mathcal{D}_{\text{pred}}} \sum_{v \in V_i} p_v^{(i)} \cdot y_v^{(i)}$
    \STATE Backpropagate: $\nabla_\theta \mathcal{L}$ and update $\theta \gets \theta - \eta \cdot \nabla_\theta \mathcal{L}$
\ENDFOR

\STATE \textbf{return} Trained scorer $\pi_\theta$
\end{algorithmic}
\end{algorithm}

\section{Figures}
\label{sec:figures}

\subsection{Training Flow}
\label{sec:training_flow}
\begin{figure}[H]
\centering
\scalebox{1.2}{
\begin{tikzpicture}[
    node distance=0.4cm and 1.0cm,
    box/.style={rectangle, draw, minimum width=2.4cm, minimum height=0.55cm, align=center, font=\normalsize},
    stage/.style={box, fill=blue!10},
    data/.style={box, fill=gray!10},
    loss/.style={box, fill=red!15},
    arrow/.style={->, >=stealth, thick},
    backward/.style={->, >=stealth, red!70, dashed, line width=0.8pt}
]

\node[data] (features) {Features $\mathbf{x}_v$};

\node[stage, below=of features] (scorer) {Scorer $\pi_\theta$};

\node[stage, below=of scorer] (risk) {Risk Computation};

\node[stage, below left=1.0cm and 1.2cm of risk] (tau_grid) {Tau Grid $\mathcal{T}$};
\node[stage, below=of tau_grid] (soft_keep_cal) {Soft Keep};
\node[stage, below=of soft_keep_cal] (coherence_cal) {Ancestor Coherence};
\node[stage, below=of coherence_cal] (validity) {Validity on Negatives};
\node[stage, below=of validity] (violation) {Violation Mapping};
\node[stage, below=of violation] (supremum) {Soft Supremum};
\node[stage, below=of supremum] (quantile) {Soft Quantile};

\node[stage, below right=1.0cm and 1.2cm of risk] (tau_grid_pred) {Tau Grid $\mathcal{T}$};
\node[stage, below=of tau_grid_pred] (soft_keep_pred) {Soft Keep};
\node[stage, below=of soft_keep_pred] (coherence_pred) {Ancestor Coherence};
\node[stage, below=of coherence_pred] (argmax) {Soft Argmax};
\node[data, below=of argmax] (final_probs) {Retention Probs $q_v$};

\node[loss, below=0.5cm of quantile, xshift=4cm] (loss_node) {Retention Loss $\mathcal{L}$};

\draw[arrow] (features) -- (scorer);
\draw[arrow] (scorer) -- (risk);
\draw[arrow] (risk) -| (tau_grid);
\draw[arrow] (tau_grid) -- (soft_keep_cal);
\draw[arrow] (soft_keep_cal) -- (coherence_cal);
\draw[arrow] (coherence_cal) -- (validity);
\draw[arrow] (validity) -- (violation);
\draw[arrow] (violation) -- (supremum);
\draw[arrow] (supremum) -- (quantile);
\draw[arrow] (quantile) -- (loss_node);

\draw[arrow] (risk) -| (tau_grid_pred);
\draw[arrow] (tau_grid_pred) -- (soft_keep_pred);
\draw[arrow] (soft_keep_pred) -- (coherence_pred);
\draw[arrow] (coherence_pred) -- (argmax);
\draw[arrow] (argmax) -- (final_probs);
\draw[arrow] (final_probs) -- (loss_node);

\draw[arrow] (quantile.east) -- ++(1.5,0) |- (argmax.west) 
    node[pos=0.2, above, xshift=-8pt, font=\small] {$\hat{\tau}_\alpha$};

\draw[backward] ([xshift=3pt]loss_node.north) -- ([xshift=3pt]final_probs.south);
\draw[backward] ([xshift=3pt]final_probs.north) -- ([xshift=3pt]argmax.south);
\draw[backward] ([xshift=3pt]argmax.north) -- ([xshift=3pt]coherence_pred.south);
\draw[backward] ([xshift=3pt]coherence_pred.north) -- ([xshift=3pt]soft_keep_pred.south);
\draw[backward] ([xshift=3pt]soft_keep_pred.north) -- ([xshift=3pt]tau_grid_pred.south);
\draw[backward] (tau_grid_pred.north) -| ([xshift=3pt]risk.east);

\draw[backward] ([xshift=-3pt]loss_node.north) -- ([xshift=-3pt]quantile.south);
\draw[backward] ([xshift=-3pt]quantile.north) -- ([xshift=-3pt]supremum.south);
\draw[backward] ([xshift=-3pt]supremum.north) -- ([xshift=-3pt]violation.south);
\draw[backward] ([xshift=-3pt]violation.north) -- ([xshift=-3pt]validity.south);
\draw[backward] ([xshift=-3pt]validity.north) -- ([xshift=-3pt]coherence_cal.south);
\draw[backward] ([xshift=-3pt]coherence_cal.north) -- ([xshift=-3pt]soft_keep_cal.south);
\draw[backward] ([xshift=-3pt]soft_keep_cal.north) -- ([xshift=-3pt]tau_grid.south);
\draw[backward] (tau_grid.north) -| ([xshift=-3pt]risk.west);

\draw[backward] ([xshift=3pt]risk.north) -- ([xshift=3pt]scorer.south) node[midway, right, font=\small, text=red!70] {$\nabla_\theta$};

\node[draw, dashed, thick, fit={(tau_grid) (soft_keep_cal) (coherence_cal) (validity) (violation) (supremum) (quantile)}, inner sep=0.15cm, label={[font=\small\bfseries]above:Calibration}] {};

\node[draw, dashed, thick, fit={(tau_grid_pred) (soft_keep_pred) (coherence_pred) (argmax) (final_probs)}, inner sep=0.15cm, label={[font=\small\bfseries]above:Prediction}] {};

\node[below=0.15cm of loss_node, font=\large, align=center] {
    \textcolor{black}{$\rightarrow$} Forward pass \quad
    \textcolor{red!70}{$\dashrightarrow$} Gradient flow to $\theta$
};

\end{tikzpicture}
}
\caption{Gradient flow through DCF training.}
\label{fig:training_flow}
\end{figure}
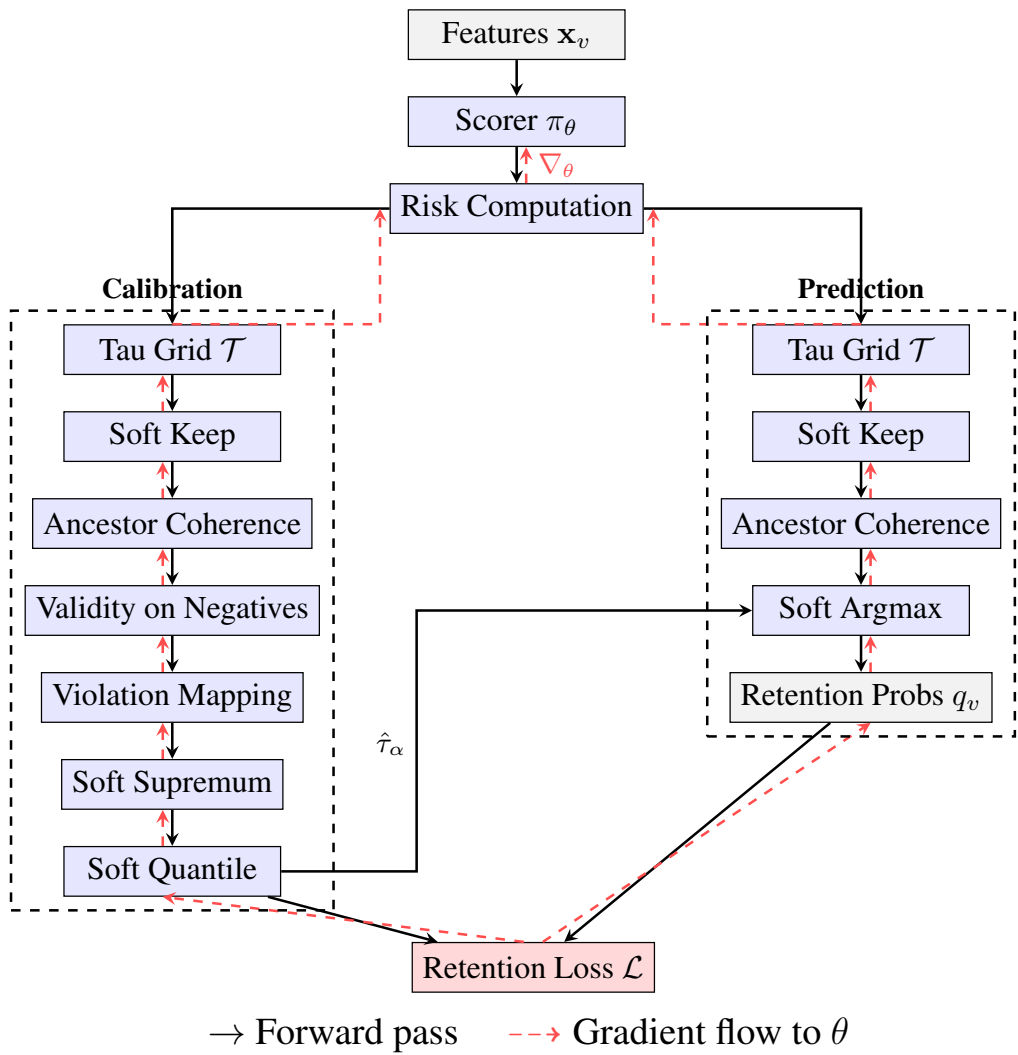

\subsection{Example Approximate Deducibility Graph}
\begin{figure}
\centering
\scalebox{1.0}{
\begin{tikzpicture}[
    node distance=0.5cm and 0.5cm,
    claim/.style={circle, draw, minimum size=0.6cm, font=\small, fill=blue!10},
    arrow/.style={->, >=stealth, thick}
]

\node[text width=6.5cm, font=\small, align=left] (problem) at (-1.5,0) {
    \textbf{Problem:} The endpoints of a segment are $(1,4)$ and $(1,10)$. What is the sum of the coordinates of the midpoint?\\[3pt]
    \textbf{Answer:} 8
};

\node[below=0cm of problem, text width=6.5cm, font=\small, align=left, anchor=north] (claims) {
    \textbf{Claims:}\\[2pt]
    0: Midpoint formula\\[1pt]
    1: $\frac{1+1}{2} = 1$\\[1pt]
    2: $\frac{4+10}{2} = 7$\\[1pt]
    3: Midpoint is $(1, 7)$\\[1pt]
    4: Sum: $1 + 7 = 8$
};

\node[claim] (c0) at (3.5, 0.3) {0};
\node[claim, below left=of c0] (c1) {1};
\node[claim, below right=of c0] (c2) {2};
\node[claim, below right=of c1] (c3) {3};
\node[claim, below=of c3] (c4) {4};

\draw[arrow] (c0) -- (c1);
\draw[arrow] (c0) -- (c2);
\draw[arrow] (c1) -- (c3);
\draw[arrow] (c2) -- (c3);
\draw[arrow] (c3) -- (c4);

\end{tikzpicture}
}
\caption{Example ADG for a MATH problem. Nodes represent atomic claims; edges indicate logical dependencies.}
\label{fig:example_graph}
\end{figure}
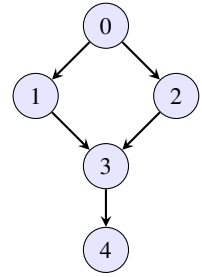

\subsection{Conformal Factuality Pipeline}
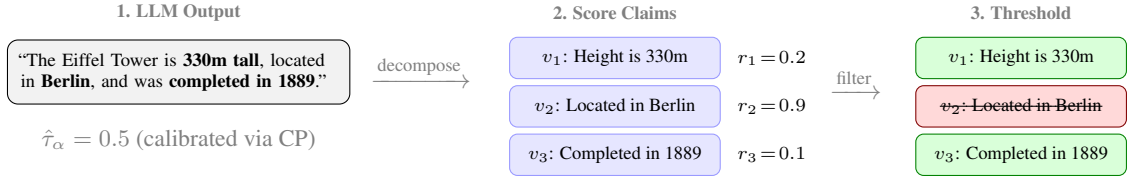
\begin{figure*}[!t]
\centering
\scalebox{1.0}{
\begin{tikzpicture}[
    node distance=0.4cm,
    box/.style={draw, rounded corners=3pt, minimum height=0.55cm, font=\scriptsize},
    kept/.style={box, fill=green!15, draw=green!50!black},
    removed/.style={box, fill=red!15, draw=red!50!black},
    neutral/.style={box, fill=blue!10, draw=blue!40},
    stage/.style={font=\scriptsize\bfseries, text=gray},
    arrow/.style={->, thick, >=stealth, gray}
]

\node[stage] (s1) at (0, 0) {1. LLM Output};
\node[below=0.12cm of s1, draw, rounded corners=4pt, fill=gray!10, text width=4.2cm, align=left, font=\scriptsize, inner sep=5pt] (output) {
    ``The Eiffel Tower is \textbf{330m tall}, located in \textbf{Berlin}, and was \textbf{completed in 1889}.''
};

\node[right=0.15cm of output, font=\normalsize, text=gray] (arr1) {$\xrightarrow{\text{decompose}}$};

\node[stage] at (5.8, 0) (s2) {2. Score Claims};

\node[neutral, minimum width=2.8cm] at (5.8, -0.6) (c1) {$v_1$: Height is 330m};
\node[right=0.1cm of c1, font=\scriptsize] (r1) {$r_1\!=\!0.2$};

\node[neutral, minimum width=2.8cm, below=0.08cm of c1] (c2) {$v_2$: Located in Berlin};
\node[right=0.1cm of c2, font=\scriptsize] (r2) {$r_2\!=\!0.9$};

\node[neutral, minimum width=2.8cm, below=0.08cm of c2] (c3) {$v_3$: Completed in 1889};
\node[right=0.1cm of c3, font=\scriptsize] (r3) {$r_3\!=\!0.1$};

\node[font=\normalsize, text=gray] at (9.0, -0.95) (arr2) {$\xrightarrow{\text{filter}}$};

\node[stage] at (11.2, 0) (s3) {3. Threshold};

\node[kept, minimum width=2.8cm] at (11.2, -0.6) (f1) {$v_1$: Height is 330m};
\node[removed, minimum width=2.8cm, below=0.08cm of f1] (f2) {\sout{$v_2$: Located in Berlin}};
\node[kept, minimum width=2.8cm, below=0.08cm of f2] (f3) {$v_3$: Completed in 1889};

\node[below=1.2cm of s1, font=\small, text=gray] (thresh) {$\hat{\tau}_\alpha = 0.5$ (calibrated via CP)};

\end{tikzpicture}
}
\caption{LLM conformal factuality. An LLM response is decomposed into atomic subclaims, each assigned a risk score. The threshold $\hat{\tau}_\alpha$ is calibrated via CP to guarantee a $1-\alpha$ factuality rate. Claims exceeding this threshold are removed (e.g., the hallucinated ``Berlin'').}
\label{fig:llm_cf_pipeline}
\end{figure*}

\section{Experimental Details}
\label{sec:experimental_details}

\subsection{Feature Descriptions}
\label{sec:feature_details}

Each claim has 31 features organized into four categories:

\paragraph{Base Scoring (1 feature).}
\texttt{frequency-score}: Self-consistency frequency across multiple LLM generations.

\paragraph{Semantic Coherence (8 features).}
\texttt{claim\_index} (position in chain), \texttt{coherent\_to\_ancestors} (logical consistency), \texttt{inference\_gap\_size} (missing steps), \texttt{has\_missing\_dependencies}, \texttt{has\_unnecessary\_dependencies}, \texttt{problem\_relevance}, \texttt{uses\_problem\_data}, and additional validation flags.

\paragraph{Domain Indicators (12 features).}
Binary flags for mathematical topics: \texttt{linear\_equations}, \texttt{quadratic\_equations}, \texttt{polynomial\_algebra}, \texttt{rational\_expressions}, \texttt{exponents\_logarithms}, \texttt{inequalities}, \texttt{sequences\_series}, \texttt{functions}, \texttt{word\_problems}, \texttt{coordinate\_geometry}, \texttt{abstract\_algebra}, \texttt{complex\_numbers}.

\paragraph{Graph Metrics (10 features).}
NetworkX-computed properties: \texttt{nx\_in\_degree}, \texttt{nx\_out\_degree}, \texttt{nx\_pagerank}, \texttt{nx\_betweenness}, \texttt{nx\_closeness}, \texttt{nx\_clustering}, \texttt{nx\_is\_source}, \texttt{nx\_is\_sink}, \texttt{nx\_reachability}, \texttt{nx\_depth\_from\_sources}.

\paragraph{FELM Feature Selection.}
For FELM, we exclude MATH-specific domain indicators and use two feature configurations based on $\alpha$:

\textbf{7 features} (for $\alpha \leq 0.08$): \texttt{frequency-score}, \texttt{coherent\_to\_ancestors}, \texttt{inference\_gap\_size}, \texttt{has\_missing\_dependencies}, \texttt{nx\_pagerank}, \texttt{nx\_reachability}, \texttt{nx\_out\_degree}.

\textbf{20 features} (for $\alpha \geq 0.09$): The 7-feature set plus \texttt{claim\_index}, \texttt{has\_unnecessary\_dependencies}, \texttt{problem\_relevance}, \texttt{uses\_problem\_data}, \texttt{domain\_science}, \texttt{domain\_reasoning}, \texttt{nx\_in\_degree}, \texttt{nx\_betweenness}, \texttt{nx\_closeness}, \texttt{nx\_clustering}, \texttt{nx\_is\_source}, \texttt{nx\_is\_sink}, \texttt{nx\_depth\_from\_sources}.

\subsection{Training Configuration}\label{sec:training-config}

Both MATH and FELM experiments share the following training setup:
\begin{itemize}[leftmargin=*,nosep]
    \item \textbf{Model}: LogisticClaimScorer---30 input features for MATH, 7 or 20 for FELM depending on $\alpha$ (see Section~\ref{sec:feature_details})
    \item \textbf{Optimizer}: Adam with learning rate 0.015, no weight-decay regularization
    \item \textbf{Training}: Up to 100 epochs with early stopping (patience\,=\,10)
    \item \textbf{Evaluation}: 20-fold cross-validation with 70/15/15 train/validation/test splits
    \item \textbf{Hyperparameter selection}: Grid search over all $\alpha$ values
\end{itemize}

\paragraph{Validation Hyperparameters.}
The surrogate-fidelity experiments in Section~\ref{sec:validating-surrogate} use fixed hyperparameters: $T_p = 0.01$, $\beta = 1.0$, $\gamma = 1.0$, $\lambda = 1.0$, $\tau_s = 0.001$, $\tau_z = 0.001$, $m = 20.0$.

\subsection{Optimized Hyperparameters}\label{sec:hyperparam-config}

Tables~\ref{tab:hyperparams_math} and~\ref{tab:hyperparams_felm} list the per-$\alpha$ hyperparameters selected by grid search for DCF.

\begin{table}[htbp]
\centering
\caption{DCF hyperparameters for MATH dataset selected via grid search.}
\label{tab:hyperparams_math}
\begin{tabular}{ccccccc}
\toprule
$\alpha$ & $\gamma$ & $\lambda$ & $\tau_s$ & $T_p$ & $\tau_z$ & lr \\
\midrule
0.03 & 4.0 & 1.35 & 1.0 & 1.0 & 0.50 & 0.015 \\
0.04 & 5.0 & 1.30 & 2.0 & 0.1 & 0.10 & 0.015 \\
0.05 & 6.0 & 1.60 & 2.0 & 0.1 & 0.10 & 0.015 \\
0.06 & 6.0 & 1.60 & 0.4 & 0.1 & 0.10 & 0.015 \\
0.07 & 1.5 & 1.70 & 0.4 & 0.1 & 0.10 & 0.015 \\
0.08 & 6.0 & 1.75 & 1.0 & 0.1 & 0.10 & 0.015 \\
0.09 & 2.0 & 1.70 & 0.4 & 0.1 & 0.01 & 0.015 \\
0.10 & 5.0 & 1.90 & 2.0 & 0.1 & 0.50 & 0.015 \\
\bottomrule
\end{tabular}
\end{table}

\begin{table}[htbp]
\centering
\caption{DCF hyperparameters for FELM dataset selected via grid search. All configurations share $\beta = 1$, $\tau_s = 0.8$, $\tau_z = 0.01$, and lr $= 0.015$.}
\label{tab:hyperparams_felm}
\begin{tabular}{cccc}
\toprule
$\alpha$ & $\gamma$ & $\lambda$ & $T_p$ \\
\midrule
0.01 & 1.50 & 1.20 & 1.0 \\
0.02 & 1.00 & 1.50 & 1.0 \\
0.03 & 1.25 & 1.30 & 1.0 \\
0.04 & 1.50 & 1.45 & 0.5 \\
0.05 & 1.25 & 1.45 & 0.5 \\
0.06 & 1.00 & 1.00 & 1.0 \\
0.07 & 1.25 & 0.90 & 1.0 \\
0.08 & 1.25 & 1.50 & 1.0 \\
0.09 & 1.00 & 1.65 & 1.0 \\
0.10 & 1.25 & 1.90 & 0.5 \\
\bottomrule
\end{tabular}
\end{table}

\subsection{Baseline $\beta_{\text{mix}}$ Optimization}\label{sec:beta-mix-opt}

For fair comparison, we optimized $\beta_{\text{mix}}$ for the frequency-based CF baseline via grid search over $\{0.0, 0.1, \ldots, 1.0\}$. Table~\ref{tab:beta_mix_combined} reports the selected values for each dataset. We use these optimized values in all comparisons to ensure DCF is evaluated against the strongest possible baseline.

\begin{table}[htbp]
\centering
\caption{Optimized $\beta_{\text{mix}}$ for the CF baseline on each dataset.}
\label{tab:beta_mix_combined}
\begin{tabular}{ccc}
\toprule
$\alpha$ & MATH & FELM \\
\midrule
0.01 & 0.0 & 0.8 \\
0.02 & 0.0 & 0.6 \\
0.03 & 0.8 & 0.5 \\
0.04 & 0.0 & 0.2 \\
0.05 & 0.0 & 0.4 \\
0.06 & 0.0 & 0.3 \\
0.07 & 0.9 & 0.4 \\
0.08 & 0.8 & 0.1 \\
0.09 & 0.2 & 0.1 \\
0.10 & 0.8 & 0.0 \\
\bottomrule
\end{tabular}
\end{table}

\subsection{Complete Results Tables}\label{sec:results-tables}

\begin{table}[htbp]
\centering
\caption{MATH dataset results: DCF vs.\ CF baseline across $\alpha$ values. Coverage target is $1{-}\alpha$. Bold coverage indicates DCF meets the target. Retention is the mean number of claims retained per problem; $\Delta$\,Ret shows percentage improvement of DCF over CF.}
\label{tab:math_results}
\small
\begin{tabular}{c|cc|ccc|c}
\toprule
$\alpha$ & \multicolumn{2}{c|}{Coverage (\%)} & \multicolumn{3}{c|}{Retention} & Meets \\
         & DCF & CF & DCF & CF & $\Delta$\,(\%) & Target? \\
\midrule
0.03 & 96.55 & 97.09 & 1.76 & 0.73 & +141.1 & No$^*$ \\
0.04 & 95.59 & 95.80 & 2.22 & 1.14 & +94.7 & No$^*$ \\
0.05 & \textbf{95.55} & 94.85 & 2.31 & 1.44 & +60.4 & Yes \\
0.06 & \textbf{94.14} & 94.10 & 3.56 & 1.74 & +104.6 & Yes \\
0.07 & \textbf{93.14} & 93.00 & 3.63 & 2.01 & +80.6 & Yes \\
0.08 & \textbf{94.14} & 91.82 & 3.69 & 2.31 & +59.7 & Yes \\
0.09 & 90.64 & 90.59 & 5.57 & 2.75 & +102.5 & No$^*$ \\
0.10 & 89.64 & 89.68 & 5.99 & 3.33 & +79.9 & No$^*$ \\
\bottomrule
\multicolumn{7}{l}{\scriptsize $^*$Near-miss: all misses are within 0.5pp of target.}
\end{tabular}
\end{table}

\begin{table}[htbp]
\centering
\caption{FELM dataset results: DCF vs.\ CF baseline across $\alpha$ values. Coverage target is $1{-}\alpha$. Bold coverage indicates DCF meets the target. Retention is the percentage of claims retained; $\Delta$\,Ret shows relative percentage improvement of DCF over CF.}
\label{tab:felm_results}
\small
\begin{tabular}{c|cc|ccc|c}
\toprule
$\alpha$ & \multicolumn{2}{c|}{Coverage (\%)} & \multicolumn{3}{c|}{Retention (\%)} & Meets \\
         & DCF & CF & DCF & CF & $\Delta$\,(\%) & Target? \\
\midrule
0.01 & \textbf{99.15} & 99.19 & 17.9 & 11.1 & +61.3 & Yes \\
0.02 & 97.75 & 98.08 & 35.7 & 22.4 & +59.4 & No$^*$ \\
0.03 & \textbf{97.32} & 97.11 & 36.4 & 29.5 & +23.4 & Yes \\
0.04 & \textbf{96.47} & 95.89 & 42.0 & 36.2 & +16.0 & Yes \\
0.05 & \textbf{95.35} & 95.10 & 46.3 & 41.5 & +11.6 & Yes \\
0.06 & \textbf{94.92} & 93.94 & 48.1 & 49.2 & $-$2.2 & Yes$^\dagger$ \\
0.07 & \textbf{94.21} & 93.13 & 49.6 & 55.0 & $-$9.8 & Yes$^\dagger$ \\
0.08 & \textbf{93.64} & 92.15 & 52.7 & 61.3 & $-$14.0 & Yes$^\dagger$ \\
0.09 & \textbf{91.39} & 90.80 & 69.0 & 66.4 & +3.9 & Yes \\
0.10 & \textbf{90.54} & 90.08 & 71.2 & 68.9 & +3.3 & Yes \\
\bottomrule
\multicolumn{7}{l}{\scriptsize $^*$DCF misses coverage target at $\alpha=0.02$ (97.75\% $<$ 98.0\%).} \\
\multicolumn{7}{l}{\scriptsize $^\dagger$DCF meets coverage but CF achieves higher retention.}
\end{tabular}
\end{table}

\section{Ablation Study Details}
\label{sec:ablation_details}

\subsection{Single-Feature Baseline Configuration}\label{sec:ablation-config}

For each single-feature baseline, we use CF with $\beta_{\text{mix}}$ optimized per-$\alpha$ via grid search over $\{0.0, 0.1, \ldots, 1.0\}$. Table~\ref{tab:beta_mix_single} shows the selected values.

\begin{table}[htbp]
\centering
\caption{Optimized $\beta_{\text{mix}}$ values for MATH single-feature baselines.}
\label{tab:beta_mix_single}
\begin{tabular}{cccc}
\toprule
$\alpha$ & NX Reachability & Claim Index & Frequency Score \\
\midrule
0.03 & 0.7 & 0.9 & 0.0 \\
0.04 & 0.9 & 1.0 & 0.0 \\
0.05 & 1.0 & 0.7 & 1.0 \\
0.06 & 0.3 & 1.0 & 0.4 \\
0.07 & 1.0 & 1.0 & 0.3 \\
0.08 & 0.2 & 0.9 & 0.1 \\
0.09 & 0.5 & 1.0 & 1.0 \\
0.10 & 0.6 & 1.0 & 0.9 \\
\bottomrule
\end{tabular}
\end{table}

\begin{table}[htbp]
\centering
\caption{Optimized $\beta_{\text{mix}}$ values for FELM single-feature baselines.}
\label{tab:beta_mix_single_felm}
\begin{tabular}{cccc}
\toprule
$\alpha$ & NX Reachability & Claim Index & Frequency Score \\
\midrule
0.01 & 0.0 & 0.5 & 0.8 \\
0.02 & 0.0 & 0.2 & 0.6 \\
0.03 & 0.2 & 0.4 & 0.5 \\
0.04 & 0.9 & 1.0 & 0.2 \\
0.05 & 0.7 & 0.9 & 0.4 \\
0.06 & 0.5 & 1.0 & 0.3 \\
0.07 & 0.0 & 1.0 & 0.4 \\
0.08 & 0.3 & 1.0 & 0.1 \\
0.09 & 0.0 & 0.8 & 0.1 \\
0.10 & 0.0 & 1.0 & 0.0 \\
\bottomrule
\end{tabular}
\end{table}

\subsection{Complete Numerical Results}\label{sec:ablation-results}

Tables~\ref{tab:ablation_full} and~\ref{tab:ablation_full_felm} provide complete retention and coverage results for MATH and FELM respectively.

\begin{table}[htbp]
\centering
\caption{MATH ablation: DCF vs.\ single-feature baselines with optimized $\beta_{\text{mix}}$.}
\label{tab:ablation_full}
\begin{tabular}{c|cc|cc|cc|cc}
\toprule
& \multicolumn{2}{c|}{\textbf{DCF (Ours)}} & \multicolumn{2}{c|}{\textbf{NX Reachability}} & \multicolumn{2}{c|}{\textbf{Claim Index}} & \multicolumn{2}{c}{\textbf{Frequency}} \\
$\alpha$ & Cov & Ret & Cov & Ret & Cov & Ret & Cov & Ret \\
\midrule
0.03 & 0.97 & \textbf{1.76} & 0.98 & 1.60 & 0.98 & 0.93 & 0.98 & 0.47 \\
0.04 & 0.96 & \textbf{2.22} & 0.97 & 2.04 & 0.97 & 1.19 & 0.97 & 0.84 \\
0.05 & 0.96 & \textbf{2.31} & 0.96 & 2.16 & 0.96 & 1.41 & 0.97 & 1.04 \\
0.06 & 0.94 & \textbf{3.56} & 0.94 & 2.87 & 0.95 & 1.86 & 0.95 & 1.49 \\
0.07 & 0.93 & \textbf{3.63} & 0.94 & 2.69 & 0.93 & 2.54 & 0.94 & 1.81 \\
0.08 & 0.94 & \textbf{3.69} & 0.93 & 3.26 & 0.93 & 2.67 & 0.93 & 1.94 \\
0.09 & 0.91 & \textbf{5.57} & 0.91 & 3.72 & 0.93 & 3.37 & 0.93 & 2.20 \\
0.10 & 0.90 & \textbf{5.99} & 0.91 & 4.18 & 0.91 & 3.62 & 0.92 & 2.21 \\
\bottomrule
\end{tabular}
\end{table}

\begin{table}[htbp]
\centering
\caption{FELM ablation: DCF vs.\ single-feature baselines with optimized $\beta_{\text{mix}}$. Retention is mean claims per problem.}
\label{tab:ablation_full_felm}
\begin{tabular}{c|cc|cc|cc|cc}
\toprule
& \multicolumn{2}{c|}{\textbf{DCF (Ours)}} & \multicolumn{2}{c|}{\textbf{NX Reachability}} & \multicolumn{2}{c|}{\textbf{Claim Index}} & \multicolumn{2}{c}{\textbf{Frequency}} \\
$\alpha$ & Cov & Ret & Cov & Ret & Cov & Ret & Cov & Ret \\
\midrule
0.01 & 0.99 & \textbf{0.72} & 0.99 & 0.17 & 0.99 & 0.24 & 0.99 & 0.51 \\
0.02 & 0.98 & \textbf{1.43} & 0.98 & 0.31 & 0.98 & 0.33 & 0.98 & 0.87 \\
0.03 & 0.97 & \textbf{1.46} & 0.97 & 0.50 & 0.97 & 0.44 & 0.97 & 1.27 \\
0.04 & 0.96 & \textbf{1.68} & 0.96 & 0.58 & 0.96 & 0.61 & 0.96 & 1.55 \\
0.05 & 0.95 & \textbf{1.85} & 0.95 & 0.73 & 0.95 & 0.72 & 0.95 & 1.71 \\
0.06 & 0.95 & 1.92 & 0.94 & 0.84 & 0.94 & 0.86 & 0.94 & \textbf{2.04} \\
0.07 & 0.94 & 1.98 & 0.93 & 0.96 & 0.93 & 1.05 & 0.93 & \textbf{2.27} \\
0.08 & 0.94 & 2.11 & 0.92 & 1.05 & 0.92 & 1.26 & 0.92 & \textbf{2.56} \\
0.09 & 0.91 & \textbf{2.76} & 0.91 & 1.22 & 0.91 & 1.31 & 0.91 & 2.72 \\
0.10 & 0.91 & \textbf{2.85} & 0.90 & 1.29 & 0.90 & 1.50 & 0.90 & 2.80 \\
\bottomrule
\end{tabular}
\end{table}

\paragraph{Key Observations (MATH).}
\begin{itemize}[leftmargin=*,nosep]
    \item DCF outperforms all single-feature baselines at every $\alpha$ value
    \item Improvement over NX Reachability (best single feature): 7--50\%
    \item Largest gains at strict coverage: +50\% at $\alpha=0.09$
    \item All methods maintain target coverage within acceptable margins
\end{itemize}

\paragraph{Key Observations (FELM).}
\begin{itemize}[leftmargin=*,nosep]
    \item DCF outperforms all single-feature baselines at low $\alpha$ (0.01--0.05) and high $\alpha$ (0.09--0.10)
    \item Frequency Score baseline wins at mid-range $\alpha$ (0.06--0.08), reflecting FELM's simpler graph structure
    \item DCF improvement over best baseline: up to +41\% at $\alpha=0.01$
    \item Frequency Score is a much stronger single-feature baseline on FELM than on MATH
\end{itemize}

\section{Case Study Details}
\label{sec:case_study_details}

\subsection{Methodology}\label{sec:case_study_methodology}

The case studies in Section~\ref{sec:case_studies} are generated using the following methodology:

\paragraph{Example Selection.} We select two representative examples from the MATH dataset at $\alpha = 0.06$:
\begin{itemize}[leftmargin=*,nosep]
    \item \textbf{Type 1 (Example 186, Claim 8)}: A TRUE claim with frequency-score = 0 that the learned model correctly retains
    \item \textbf{Type 2 (Example 190, Claim 2)}: A FALSE claim with frequency-score = 5 that the learned model correctly rejects
\end{itemize}

\paragraph{Evaluation Protocol.} We use 20-fold cross-validation matching the main experimental setup:
\begin{enumerate}[leftmargin=*,nosep]
    \item For each fold, train the learned logistic claim scorer on the training set
    \item Compute calibration quantiles for both learned and baseline methods
    \item Generate prediction sets for the test examples using both methods
    \item Record which claims are retained in each fold's prediction set
\end{enumerate}

\paragraph{Baseline Configuration.} The frequency-based CF baseline uses $\beta_{\text{mix}}$ optimized per-$\alpha$ via grid search over $\{0.0, 0.1, \ldots, 1.0\}$. At $\alpha = 0.06$, the optimal value is $\beta_{\text{mix}} = 0.4$ (see Table~\ref{tab:beta_mix_single}).

\paragraph{Majority Voting.} For visualization, we use majority voting across folds: a claim is considered ``retained'' if it appears in the prediction set in more than 50\% of folds ($>10$ out of 20 folds). This provides stable results despite fold-to-fold variation.

\paragraph{Feature Contributions.} Model weights are averaged across the 20 folds to compute stable feature contributions. The contribution of each feature is computed as $\text{weight} \times \text{value}$, and the total learned score is the sum of all contributions plus bias.

\subsection{Feature Contribution Tables}

Tables~\ref{tab:case_study_retain} and~\ref{tab:case_study_reject} show feature contributions for the case studies. Only features with non-zero values are shown; the remaining features (primarily domain indicators not applicable to the specific problem) have value 0 and contribute nothing to the score.

\begin{table}[htbp]
\centering
\caption{Case Study 1 (Example 186, Claim 8): Feature contributions for valid claim with zero frequency.}
\label{tab:case_study_retain}
\begin{tabular}{lrrr}
\toprule
Feature & Value & Weight & Contribution \\
\midrule
\texttt{nx\_reachability} & 3.00 & 0.272 & +0.815 \\
\texttt{claim\_index} & 8.00 & 0.098 & +0.780 \\
\texttt{nx\_in\_degree} & 2.00 & 0.135 & +0.269 \\
\texttt{quadratic\_equations} & 1.00 & 0.229 & +0.229 \\
\texttt{nx\_out\_degree} & 1.00 & 0.176 & +0.176 \\
\texttt{problem\_relevance} & 1.00 & 0.144 & +0.144 \\
\texttt{coherent\_to\_ancestors} & 1.00 & 0.110 & +0.110 \\
\texttt{nx\_betweenness} & 0.22 & 0.292 & +0.064 \\
\texttt{uses\_problem\_data} & 0.50 & 0.104 & +0.052 \\
\texttt{frequency-score} & 0.00 & 0.021 & +0.000 \\
\midrule
\multicolumn{3}{l}{\textbf{Total Learned Score}} & \textbf{2.66} \\
\multicolumn{3}{l}{\textbf{Baseline Score (freq-score)}} & 0.00 \\
\multicolumn{3}{l}{\textbf{Prediction Set Retention}} & Learned: 9/12, Baseline: 0/12 \\
\bottomrule
\end{tabular}
\end{table}

\begin{table}[htbp]
\centering
\caption{Case Study 2 (Example 190, Claim 2): Feature contributions for invalid claim with high frequency.}
\label{tab:case_study_reject}
\begin{tabular}{lrrr}
\toprule
Feature & Value & Weight & Contribution \\
\midrule
\texttt{nx\_is\_source} & 1.00 & $-$0.253 & \textbf{$-$0.253} \\
\texttt{word\_problems} & 1.00 & 0.227 & +0.227 \\
\texttt{claim\_index} & 2.00 & 0.098 & +0.195 \\
\texttt{nx\_in\_degree} & 1.00 & 0.135 & +0.135 \\
\texttt{coherent\_to\_ancestors} & 1.00 & 0.110 & +0.110 \\
\texttt{frequency-score} & 5.00 & 0.021 & +0.107 \\
\texttt{nx\_pagerank} & 0.13 & $-$0.249 & $-$0.033 \\
\texttt{nx\_closeness} & 0.22 & 0.093 & +0.021 \\
\midrule
\multicolumn{3}{l}{\textbf{Total Learned Score}} & \textbf{0.54} \\
\multicolumn{3}{l}{\textbf{Baseline Score (freq-score)}} & 5.00 \\
\multicolumn{3}{l}{\textbf{Prediction Set Retention}} & Learned: 0/7, Baseline: 7/7 \\
\bottomrule
\end{tabular}
\end{table}

\subsection{Analysis}

\paragraph{Case 1: Retaining Valid Claims with Zero Frequency.}
This case study appears in the main text (Section~\ref{sec:case_studies}). Graph connectivity features (\texttt{nx\_reachability} +0.82, \texttt{nx\_in\_degree} +0.27) and claim position (\texttt{claim\_index} +0.78) compensate for zero frequency-score, enabling retention of a valid intermediate reasoning step that the baseline would reject entirely.
\paragraph{Case 2: Rejecting Invalid Claims with High Frequency.}
Claim 2 (``The integer represented by the input is greater than 0'') is a hallucination with frequency$=5.0$---it appeared consistently across regenerations despite being invalid. The frequency-based baseline retains it (and all 7 claims in this semantically incoherent output).

DCF rejects this claim by recognizing its isolated graph position. Despite high frequency-score contributing +0.11, the learned model assigns a large negative contribution to \texttt{nx\_is\_source} ($-$0.25), correctly penalizing this claim that lacks proper ancestors in the reasoning graph. The total learned score (0.54) falls below the threshold, and DCF correctly rejects all 7 claims in this output.

This demonstrates the complementary case to Section~\ref{sec:case_studies}: high frequency alone does not indicate validity when graph structure signals incoherence.

\begin{figure}[htbp]
\centering
\scalebox{0.7}{
\begin{tikzpicture}[
    node distance=0.5cm and 0.6cm,
    claim/.style={circle, draw, minimum size=0.7cm, font=\small, fill=red!20, line width=1.5pt},
    highlight/.style={circle, draw, minimum size=0.7cm, font=\small, fill=yellow!40, line width=2.5pt, draw=orange},
    arrow/.style={->, >=stealth, thick}
]

\node[text width=7cm, font=\normalsize, align=justify, draw, fill=blue!5, anchor=north west] (problem) at (-6,0) {
    \textbf{Problem:} A rectangular patio has an area of 180 square feet and a perimeter of 54 feet. What is the length of the diagonal (in feet) squared?\\[2pt]
    \textbf{Answer:} 585
};

\node[below=0.4cm of problem.south west, anchor=north west, text width=7cm, font=\small, align=left, draw, fill=white] (claims) {
    \textbf{Claims (all FALSE):}\\
    0: Input consists of digits `3', `6', `9'\\
    1: Input represents integer 369\\
    2: Integer from input is $> 0$ \textcolor{orange}{$\leftarrow$ freq=5.0}\\
    3: Integer 369 is odd\\
    4: Integer 369 is divisible by 3\\
    5: Integer 369 is divisible by 9\\
    6: $369 = 9 \times 41$\\[3pt]
    \textit{``369'' never appears in problem}
};

\node[below=0.4cm of claims.south west, anchor=north west, text width=7cm, font=\normalsize, align=justify, draw, fill=red!5] (summary) {
    \textbf{Performance:} \textcolor{green!70!black}{Learned: 6/7 rejected}\\
    \textcolor{red!70!black}{Baseline: 0/7 rejected}\\
    \textcolor{orange}{Highlighted: Claim 2 (freq=5.0)}
};

\node[claim] (c0) at (3.6, -1) {0};
\node[claim, below=0.6cm of c0] (c1) {1};
\node[highlight, below=0.6cm of c1] (c2) {2};
\node[claim, below left=0.6cm and 0.8cm of c1] (c3) {3};
\node[claim, below=0.6cm of c3] (c4) {4};
\node[claim, below right=0.6cm and 0.8cm of c1] (c6) {6};
\node[claim, below=0.6cm of c6] (c5) {5};

\draw[arrow] (c0) -- (c1);
\draw[arrow] (c1) -- (c2);
\draw[arrow] (c1) -- (c3);
\draw[arrow] (c1) -- (c4);
\draw[arrow] (c1) -- (c6);
\draw[arrow] (c6) -- (c5);

\node[below=1.2cm of c4, font=\scriptsize, align=center] (legend) {
    \textcolor{red!70!black}{\textbf{$\bullet$}} All FALSE\\
    \textcolor{orange}{\textbf{$\bullet$}} Claim 2
};

\end{tikzpicture}
}
\caption{Freq. vs. Learned ADG Comparison (Incorrect Claims)}
\label{fig:case_study_reject}
\end{figure}
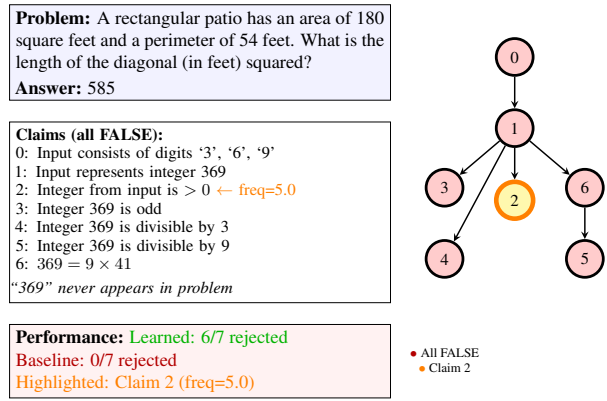

\section{SHAP Feature Importance Analysis}
\label{sec:shap_analysis}

\subsection{Methodology}\label{sec:shap-methodology}

To ensure robust feature importance estimates, we aggregate SHAP values across multiple coverage levels and training splits. We analyze models trained at $\alpha \in \{0.03, 0.04, 0.05, 0.06, 0.07, 0.08, 0.09, 0.10\}$. For each $\alpha$, we train 5 independent models using different train/validation splits (80/20), yielding 40 total models (8 $\alpha$ values $\times$ 5 splits).

\paragraph{SHAP Computation.}
\begin{itemize}[leftmargin=*,nosep]
    \item Method: Kernel SHAP (model-agnostic explainer)
    \item Background data: 100 samples randomly selected from all claims
    \item Evaluation samples: Maximum 500 claims per model
    \item For LogisticClaimScorer, we extract linear weights $w$ and bias $b$ to create prediction function $f(x) = xw + b$
\end{itemize}

\paragraph{Aggregation Procedure.}
\begin{enumerate}[leftmargin=*,nosep]
    \item Collect SHAP values from all 40 models (each contributing up to 500 claim samples)
    \item Pool all SHAP values per feature across models ($\sim$20,000 values per feature)
    \item Compute mean absolute SHAP value: $\text{importance}_i = \text{mean}(|\text{SHAP}_i|)$
    \item Compute standard deviation to quantify variance across models
    \item Rank features by mean importance
\end{enumerate}

\subsection{Results}\label{sec:shap-results}

Table~\ref{tab:shap_results} shows the top 10 features by mean $|\text{SHAP}|$ value.

\begin{table}[htbp]
\centering
\caption{Top 10 features by SHAP importance aggregated across 40 models.}
\label{tab:shap_results}
\begin{tabular}{clcc}
\toprule
Rank & Feature & Mean $|\text{SHAP}|$ & Std \\
\midrule
1 & \texttt{nx\_reachability} & 0.602 & 0.428 \\
2 & \texttt{claim\_index} & 0.445 & 0.303 \\
3 & \texttt{frequency-score} & 0.180 & 0.126 \\
4 & \texttt{inference\_gap\_size} & 0.124 & 0.248 \\
5 & \texttt{nx\_out\_degree} & 0.060 & 0.081 \\
6 & \texttt{nx\_in\_degree} & 0.038 & 0.056 \\
7 & \texttt{nx\_is\_source} & 0.034 & 0.035 \\
8 & \texttt{problem\_relevance} & 0.034 & 0.032 \\
9 & \texttt{uses\_problem\_data} & 0.034 & 0.041 \\
10 & \texttt{quadratic\_equations} & 0.022 & 0.020 \\
\bottomrule
\end{tabular}
\end{table}

\paragraph{Key Findings.}
\begin{itemize}[leftmargin=*,nosep]
    \item \textbf{Graph structure dominates}: \texttt{nx\_reachability} (0.602) contributes 3.3$\times$ more than \texttt{frequency-score} (0.180)
    \item \textbf{Position matters}: \texttt{claim\_index} (0.445) contributes 2.5$\times$ more than frequency
    \item \textbf{Coherence features rank highly}: \texttt{inference\_gap\_size} (0.124) captures missing reasoning steps
    \item \textbf{Mathematical content has lower impact}: Domain indicators individually contribute less, but collectively capture topic-specific patterns
\end{itemize}

Figure~\ref{fig:shap_beeswarm_full} shows the complete SHAP beeswarm plot with all 30 features.

\begin{figure}[htbp]
\centering
\includegraphics[width=0.7\textwidth]{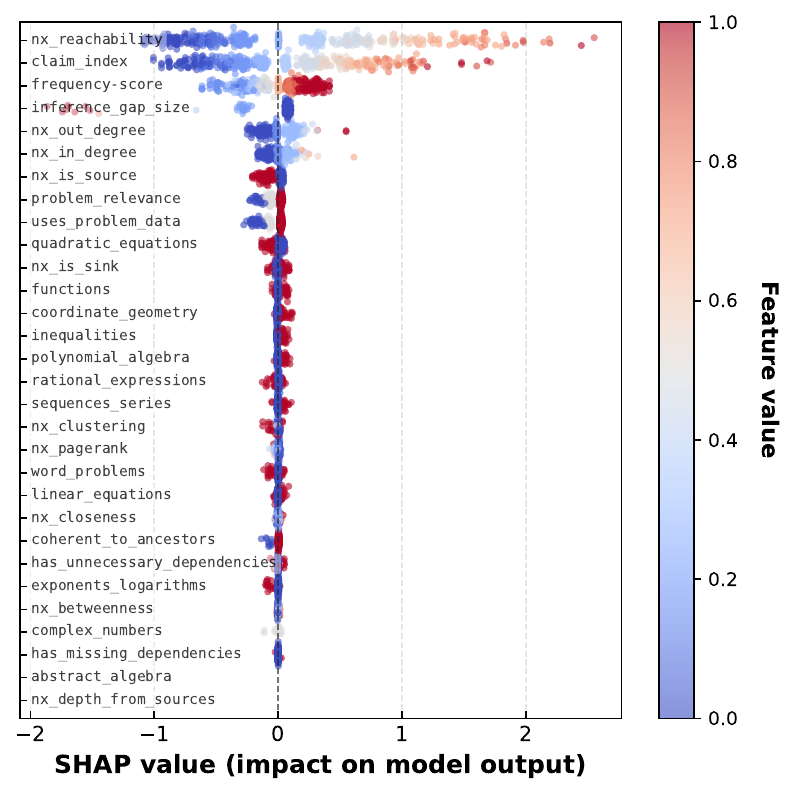}
\caption{MATH SHAP beeswarm plot across 40 models. Each dot represents one claim; color indicates normalized feature value (blue=low, red=high). Graph structure features dominate over frequency-based scoring.}
\label{fig:shap_beeswarm_full}
\end{figure}

\subsection{FELM SHAP Analysis}\label{sec:felm-shap}

For FELM, we perform separate SHAP analyses for the two feature configurations used at different $\alpha$ levels: 7 features for $\alpha \in [0.01, 0.08]$ and 20 features for $\alpha \in [0.09, 0.10]$.

\paragraph{7-Feature Models ($\alpha \leq 0.08$).}
The 7-feature configuration uses: \texttt{frequency-score}, \texttt{coherent\_to\_ancestors}, \texttt{inference\_gap\_size}, \texttt{has\_missing\_dependencies}, \texttt{nx\_pagerank}, \texttt{nx\_reachability}, and \texttt{nx\_out\_degree}. Figure~\ref{fig:shap_felm_7} shows that \texttt{frequency-score} dominates (0.995), with \texttt{nx\_reachability} (0.154) and \texttt{inference\_gap\_size} (0.093) as secondary contributors.

\paragraph{20-Feature Models ($\alpha \geq 0.09$).}
At higher $\alpha$ values, we use the full 20-feature set including additional graph metrics and coherence features. Figure~\ref{fig:shap_felm_20} shows that \texttt{frequency-score} remains dominant (0.941), but \texttt{claim\_index} (0.209) and \texttt{nx\_reachability} (0.149) gain importance with the expanded feature set.

\begin{figure}[htbp]
\centering
\begin{subfigure}[b]{0.48\textwidth}
    \centering
    \includegraphics[width=\textwidth]{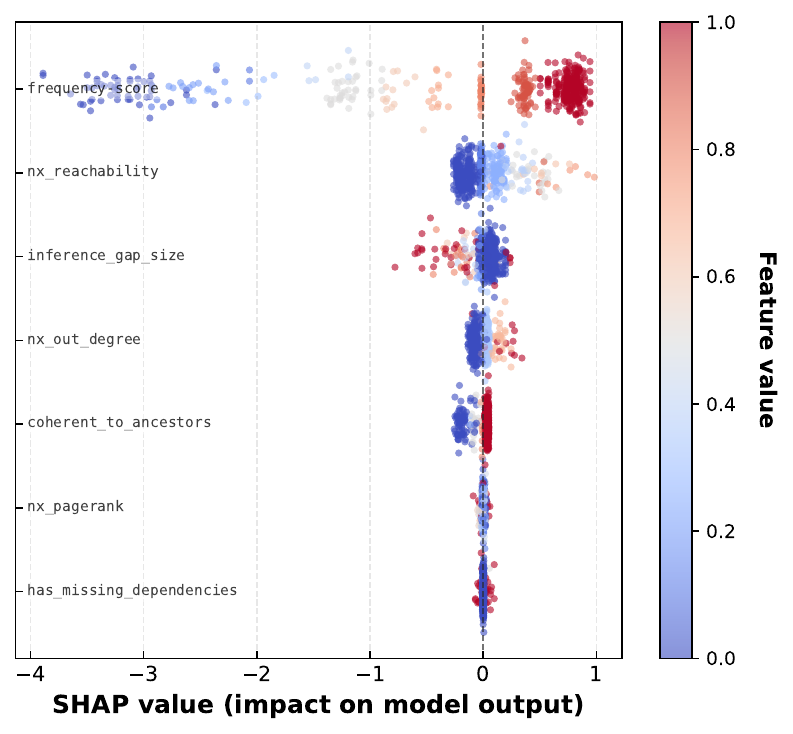}
    \caption{7-feature models ($\alpha \leq 0.08$).}
    \label{fig:shap_felm_7}
\end{subfigure}
\hfill
\begin{subfigure}[b]{0.48\textwidth}
    \centering
    \includegraphics[width=\textwidth]{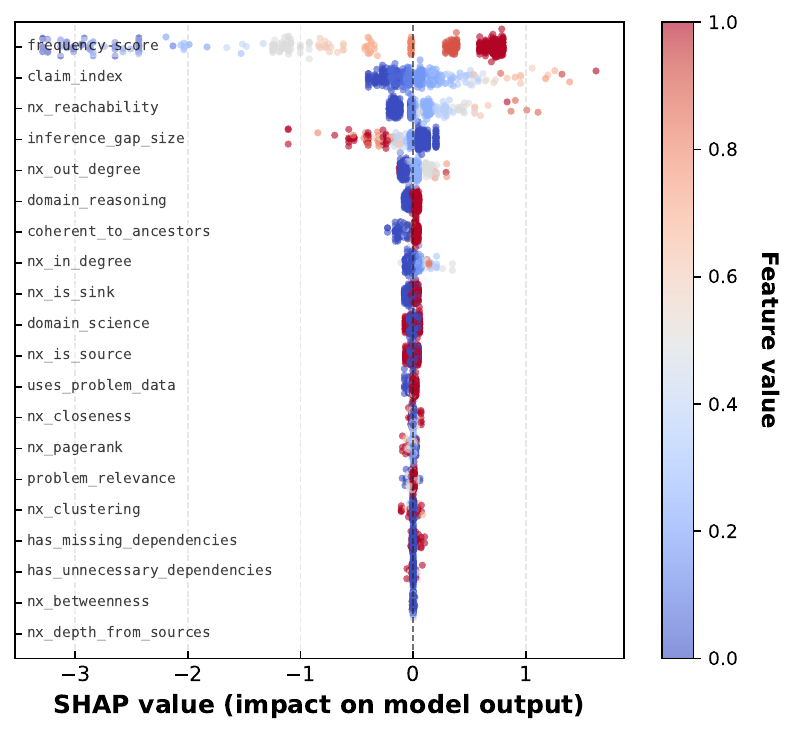}
    \caption{20-feature models ($\alpha \geq 0.09$).}
    \label{fig:shap_felm_20}
\end{subfigure}
\caption{FELM SHAP beeswarm plots. Left: 7-feature configuration used at stricter coverage levels. Right: 20-feature configuration used at relaxed coverage levels. \texttt{frequency-score} dominates in both, reflecting FELM's simpler reasoning chains where self-consistency is more discriminative.}
\label{fig:shap_felm}
\end{figure}

\paragraph{MATH vs.\ FELM Comparison.}
The key difference between datasets is the relative importance of \texttt{frequency-score}: it ranks 3rd on MATH (0.180) but 1st on FELM (0.941--0.995). This reflects FELM's simpler reasoning chains (4.0 vs.\ 7.3 claims per problem), where self-consistency frequency is more discriminative. On MATH, graph structure features (\texttt{nx\_reachability}, \texttt{claim\_index}) provide stronger signal for navigating complex dependency graphs.

\section{Score Distribution Analysis}
\label{sec:score_analysis}

To understand how DCF differentiates true from false claims, we analyze score distributions using pooled z-score normalization.

\subsection{Methodology}\label{sec:score-methodology}

For each $\alpha$ value, we:
\begin{enumerate}[leftmargin=*,nosep]
    \item Collect scores from all 20 CV folds for both learned and baseline models
    \item Apply z-score normalization pooled across both methods for comparability
    \item Compute separation as: $\text{sep} = \mu_{\text{true}} - \mu_{\text{false}}$
    \item Compute Cohen's $d$ effect size and distribution overlap
\end{enumerate}

\subsection{Results}\label{sec:score-results}

Figure~\ref{fig:score_distributions_full} shows the 2$\times$2 comparison at $\alpha=0.05$.

\begin{figure}[htbp]
\centering
\begin{subfigure}[b]{0.45\textwidth}
    \includegraphics[width=\textwidth]{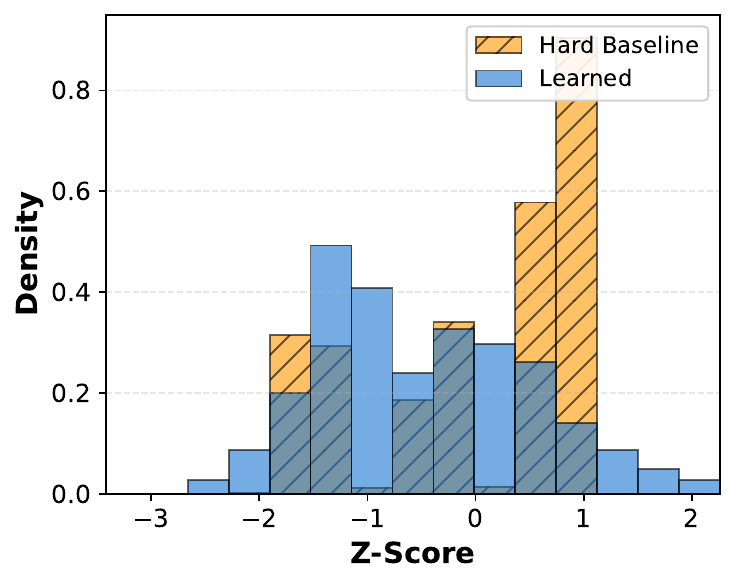}
    \caption{True claims: learned vs baseline}
\end{subfigure}
\hfill
\begin{subfigure}[b]{0.45\textwidth}
    \includegraphics[width=\textwidth]{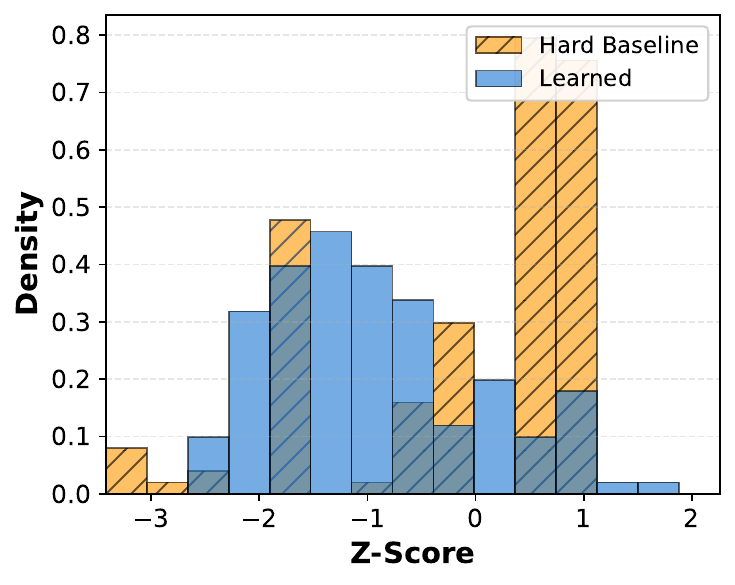}
    \caption{False claims: learned vs baseline}
\end{subfigure}

\vspace{0.3cm}

\begin{subfigure}[b]{0.45\textwidth}
    \includegraphics[width=\textwidth]{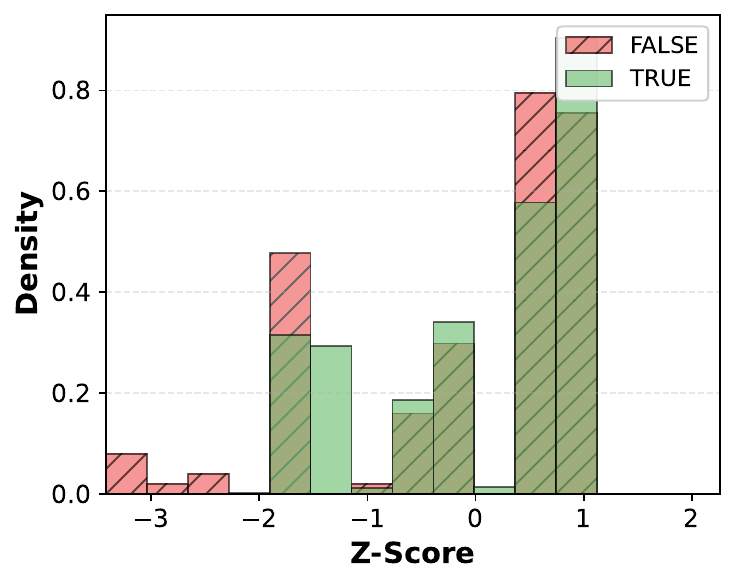}
    \caption{Baseline: true vs false separation}
\end{subfigure}
\hfill
\begin{subfigure}[b]{0.45\textwidth}
    \includegraphics[width=\textwidth]{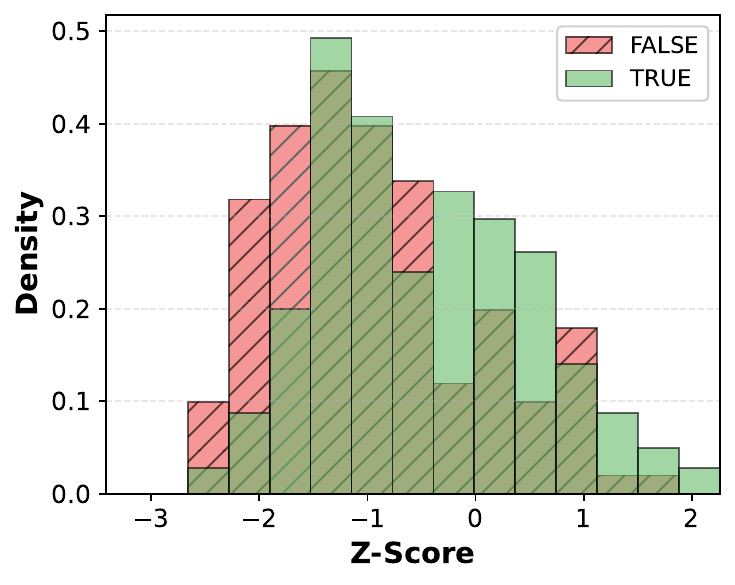}
    \caption{Learned: true vs false separation}
\end{subfigure}
\caption{Score distributions at $\alpha = 0.05$. Top row: method comparison for true/false claims. Bottom row: within-method discrimination. The learned model achieves separation 0.450 vs.\ baseline's 0.141 (3.2$\times$).}
\label{fig:score_distributions_full}
\end{figure}

Table~\ref{tab:separation_full} shows separation metrics across all $\alpha$ values.

\begin{table}[htbp]
\centering
\caption{Score separation metrics across $\alpha$ values. DCF optimizes retention under coverage constraints, not separation directly, explaining the variability.}
\label{tab:separation_full}
\begin{tabular}{ccccc}
\toprule
$\alpha$ & Separation & Ratio vs.\ Baseline & Cohen's $d$ & Overlap \\
\midrule
0.03 & 0.170 & 1.21$\times$ & 0.298 & 0.926 \\
0.04 & 0.040 & 0.28$\times$ & 0.046 & 0.962 \\
0.05 & 0.450 & 3.18$\times$ & 0.468 & 0.939 \\
0.06 & 0.056 & 0.40$\times$ & 0.057 & 0.954 \\
0.07 & 0.535 & 3.80$\times$ & 0.514 & 0.882 \\
0.08 & 0.434 & 3.05$\times$ & 0.574 & 0.927 \\
0.09 & $-$0.184 & -- & $-$0.192 & 0.933 \\
0.10 & 0.111 & 0.79$\times$ & 0.124 & 0.945 \\
\bottomrule
\end{tabular}
\end{table}

\paragraph{Interpretation.} Separation varies substantially---and can even be negative---because DCF optimizes for coverage-constrained retention, not distributional separation. At $\alpha=0.09$, DCF achieves +103\% higher retention despite negative separation, demonstrating that conformal prediction succeeds through calibrated threshold placement rather than global score discrimination. The quantile-based approach requires only that the threshold correctly partitions claims at the decision boundary, not that true and false claims be globally separated.

\end{document}